\documentclass[runningheads]{llncs}
\usepackage{graphicx}
\usepackage{epstopdf}
\usepackage{amsmath,amssymb} 
\usepackage{color}
\usepackage[width=122mm,left=12mm,paperwidth=146mm,height=193mm,top=12mm,paperheight=217mm]{geometry}

\usepackage[ruled]{algorithm}
\usepackage{algpseudocode}

\newcommand{\ud}{\,\mathrm{d}}

\newcommand{\R}{\mathbb{R}}

\newcommand{\mean}[1]{\mbox{avg}({#1})}

\newcommand{\comment}[1]{ }

\begin{document}
\pagestyle{headings}

\mainmatter
\def\ECCV16SubNumber{1485}  

\title{Coarse-to-Fine Segmentation With Shape-Tailored Scale Spaces}




\author{Ganesh Sundaramoorthi$^{1}$, Naeemullah Khan$^{1}$, Byung-Woo Hong$^{2}$}
\institute{$^{1}$KAUST, Saudi Arabia \quad $^2$Chung-Ang University, Korea}

\maketitle

\begin{abstract}
  We formulate a general energy and method for segmentation that is
  designed to have preference for segmenting the coarse structure over
  the fine structure of the data, without smoothing across boundaries
  of regions. The energy is formulated by considering data terms at a
  continuum of scales from the scale space computed from the Heat
  Equation within regions, and integrating these terms over all
  time. We show that the energy may be approximately optimized without
  solving for the entire scale space, but rather solving
  time-independent linear equations at the native scale of the image,
  making the method computationally feasible. We provide a
  multi-region scheme, and apply our method to motion
  segmentation. Experiments on a benchmark dataset shows that our
  method is less sensitive to clutter or other undesirable fine-scale
  structure, and leads to better performance in motion segmentation.


\end{abstract}

\begin{figure}
  \centering
  \vspace{-0.75cm}
  {\footnotesize Without Coarse-to-Fine Approach}\\
  \includegraphics[width=.15\linewidth,height=.1\linewidth]{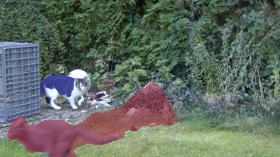}
  \includegraphics[width=.15\linewidth,height=.1\linewidth]{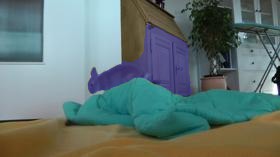}
  \includegraphics[width=.15\linewidth,height=.1\linewidth]{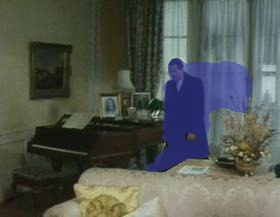}
  \includegraphics[width=.15\linewidth,height=.1\linewidth]{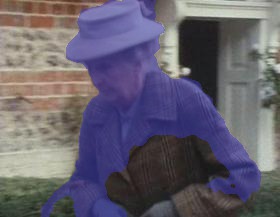}
  \includegraphics[width=.15\linewidth,height=.1\linewidth]{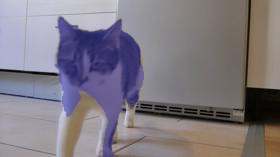}
  \includegraphics[width=.15\linewidth,height=.1\linewidth]{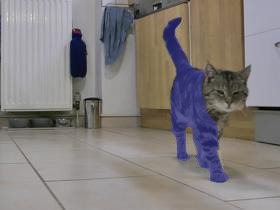}\\
  {\footnotesize With Our Coarse-to-Fine Approach}\\
  \includegraphics[width=.15\linewidth,height=.1\linewidth]{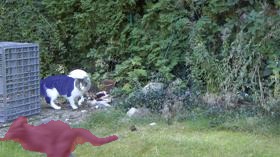}
  \includegraphics[width=.15\linewidth,height=.1\linewidth]{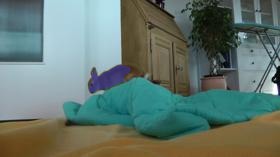}
  \includegraphics[width=.15\linewidth,height=.1\linewidth]{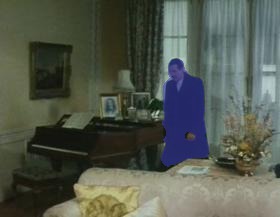}
  \includegraphics[width=.15\linewidth,height=.1\linewidth]{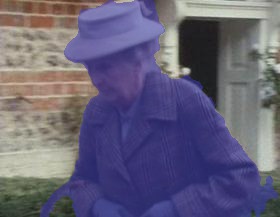}
  \includegraphics[width=.15\linewidth,height=.1\linewidth]{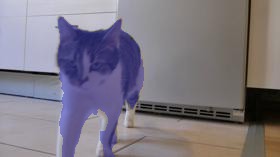}
  \includegraphics[width=.15\linewidth,height=.1\linewidth]{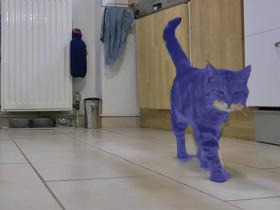}\\
  \caption{Our novel reformulation of existing data terms in
    segmentation based on scale spaces defined within regions of the
    segmentation, and integrated over a continuum of scales, leads to
    a scheme that exhibits a coarse-to-fine property.  This leads to
    less sensitivity of the method to clutter while retaining boundary
    accuracy.}
  \label{fig:motivation}
  \vspace{-0.75cm}
\end{figure}

\section{Introduction}

Segmentation of images and videos using low-level cues plays a key
role in computer vision. An image consists of many different
structures at different \emph{scales}, and thus the notion of
\emph{scale space} \cite{koenderink1984structure}, which consists of
blurs of the image at all degrees, has been central to computer
vision.  The need for incorporating scale space in segmentation is
well-recognized \cite{perona1990scale}. Further, there is evidence
from human visual studies (e.g., \cite{hegde2008time,neri2011coarse})
that the coarse scale, i.e., from high levels of blurring, is
predominantly processed before the fine scale. This
\emph{coarse-to-fine} principle has led to many efficient algorithms
that are able to capture the coarse structure of the solution, which
is often most important in computer vision. Therefore, it is natural
for segmentation algorithms to use scale space and operate in a
coarse-to-fine fashion.

Existing methods for segmentation that incorporate scale have either
one of the following limitations. First, most segmentation methods
(e.g.,
\cite{bresson2006multiscale,kokkinos2009texture,maire2013progressive,arbelaez2014multiscale})
based on scale spaces consider scale spaces that are globally computed
on the whole image, which does not capture the fact that there exist
multiple \emph{regions} of the segmentation at different scales, and
this could lead to the removal and/or displacement of important
structures in the image, for instance, when large structures are
blurred across small ones, leading to an inaccurate segmentation. Second,
algorithms that use a coarse-to-fine principle (e.g.,
\cite{blake1987visual,mobahi2015coarse}) do so \emph{sequentially}
(see Figure~\ref{fig:parallel_ctf}) so that the algorithm operates at
the coarser scale and then uses the result to initialize computation at a
finer scale. While this warm start may influence the finer scale
result, there is no guarantee that the coarse structure of the
segmentation is preserved in the final solution.

{\bf Contributions}: In this paper, we develop an algorithm that
simultaneously addresses these two issues. {\bf 1.} Specifically, we
formulate a novel multi-label energy for segmentation that integrates
over a \emph{continuum} scales of scale spaces that are defined within
regions of the segmentation, referred to as a \emph{Shape-Tailored
  Scale Space}, thus preventing removal or displacement of important
structures. By integrating over a continuum of scales of the scale
space determined by the Heat Equation, we show that this energy has
preference to coarse structure of the data without ignoring the fine
structure. {\bf 2.}  Further, we show that the optimization of the
energy operates in a \emph{parallel} coarse-to-fine fashion (see
Figure~\ref{fig:parallel_ctf}). In particular, it considers a
continuum of scales together, and it is initially dominated by the
coarse structure of the data, then moves to segment finer structure of
the data, while preserving the structure obtained from the coarse
scale of the data. {\bf 3.} We apply our algorithm to the problem of
segmenting objects in video by motion, and show that we improve an
existing method on a benchmark dataset by merely changing the data
term in the energy to incorporate our ideas.

\def\fHeight{0.8in}
\begin{figure}[tb]
  \centering
  {\footnotesize Sequential Coarse-to-Fine}\\
  \includegraphics[clip,trim=0 10 10
  0,totalheight=\fHeight]{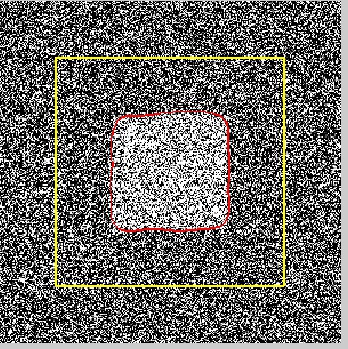}
  \includegraphics[clip,trim=0 10 10
  0,totalheight=\fHeight]{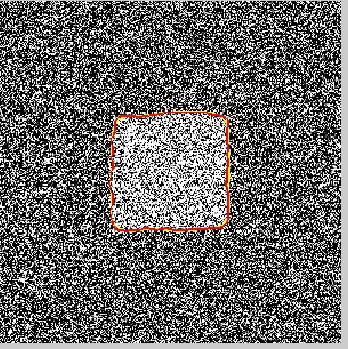}
  \includegraphics[clip,trim=0 10 10
  0,totalheight=\fHeight]{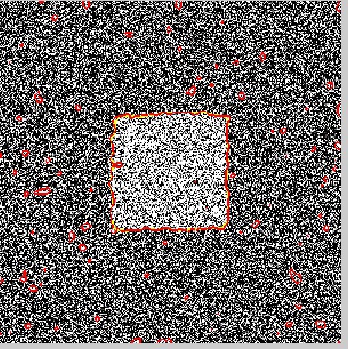}
  \includegraphics[clip,trim=0 10 10
  0,totalheight=\fHeight]{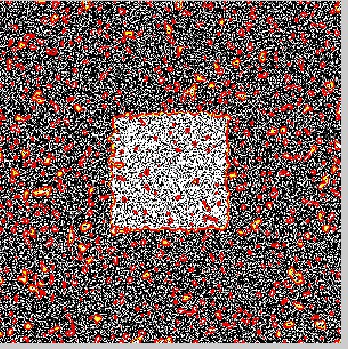}
  \includegraphics[clip,trim=0 10 10
  0,totalheight=\fHeight]{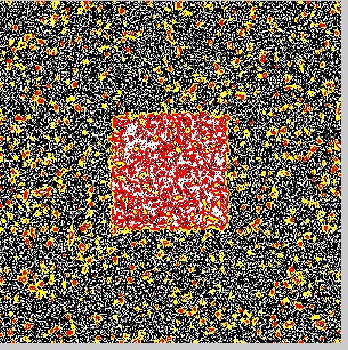}
  \\
%
  {\footnotesize Parallel Coarse-to-Fine (Ours)}\\
  \includegraphics[clip,trim=125 50 110 40,totalheight=\fHeight]{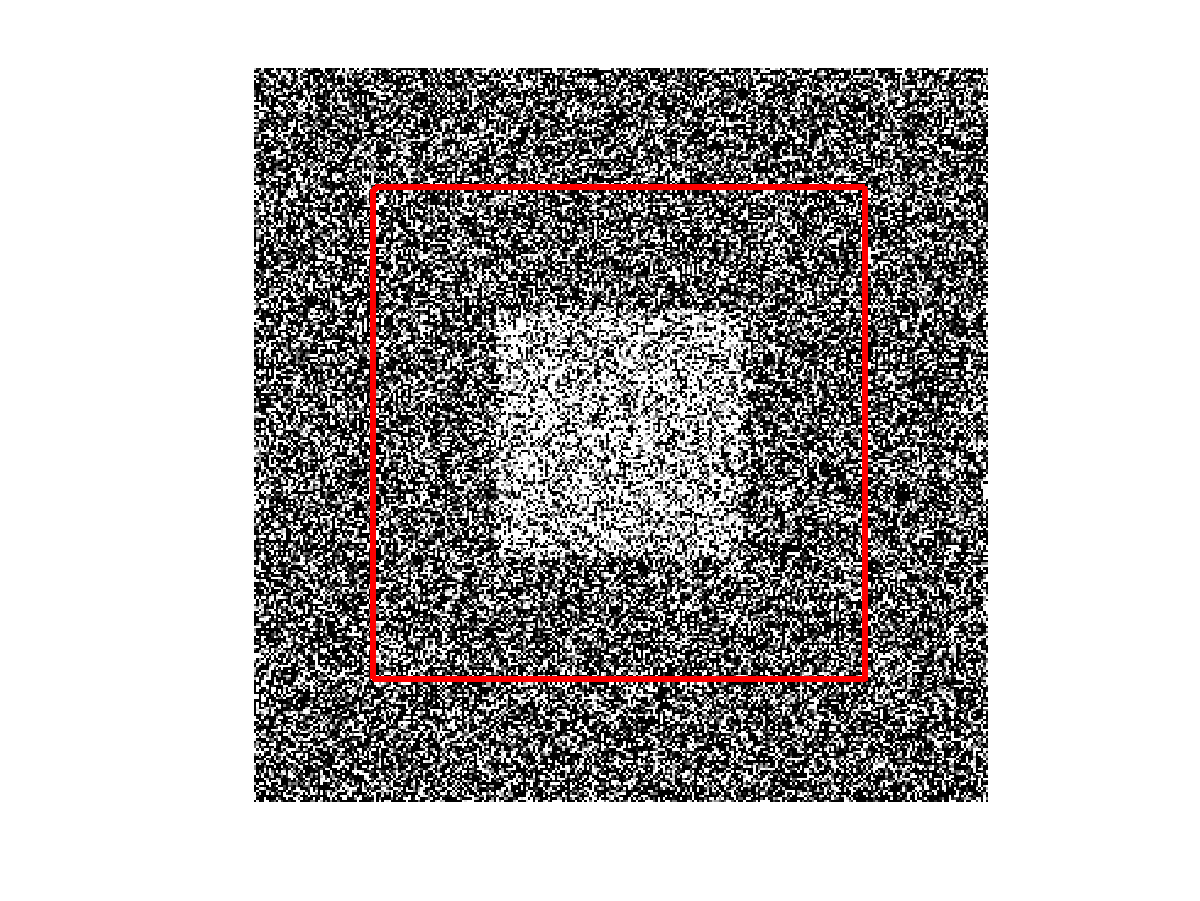}
  \includegraphics[clip,trim=125 50 110 40,totalheight=\fHeight]{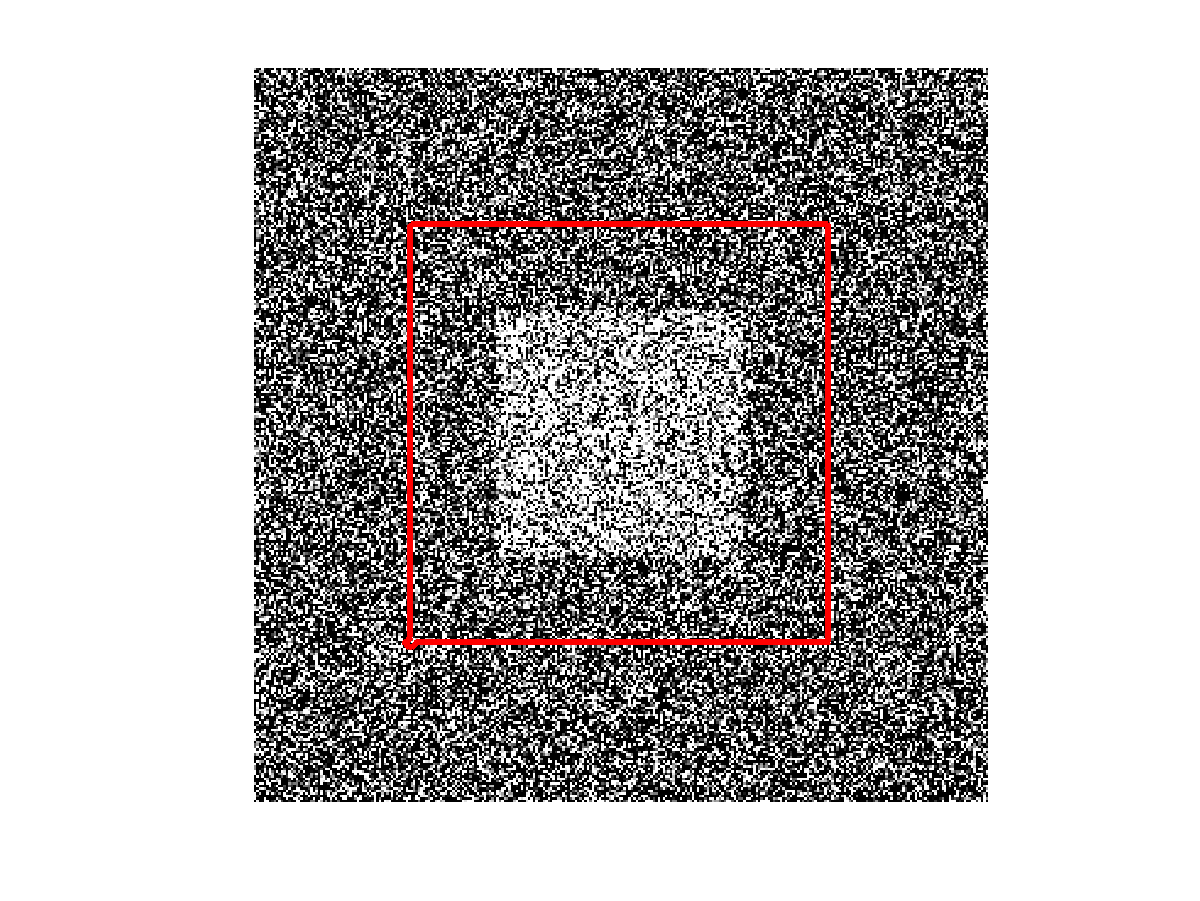}
  \includegraphics[clip,trim=125 50 110
  40,totalheight=\fHeight]{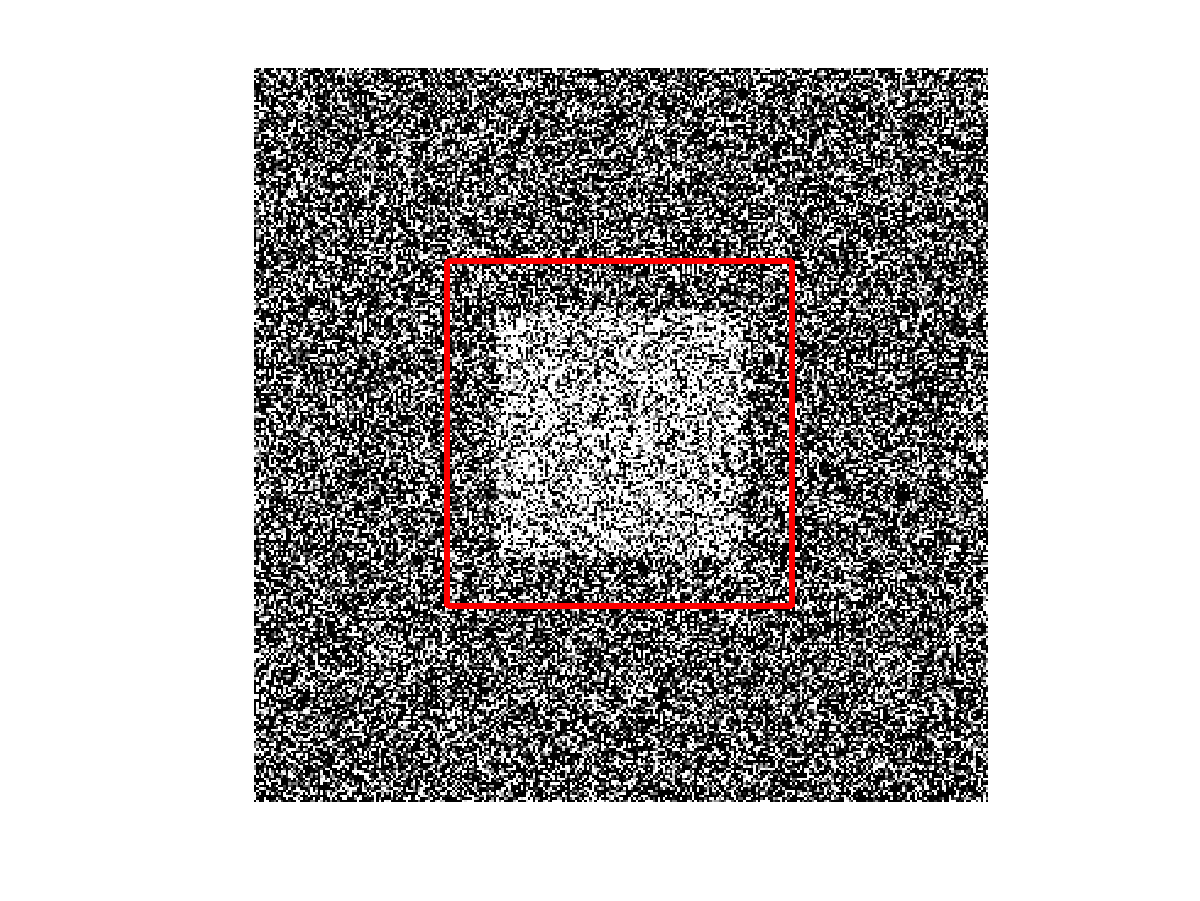}
  \includegraphics[clip,trim=125 50 110
  40,totalheight=\fHeight]{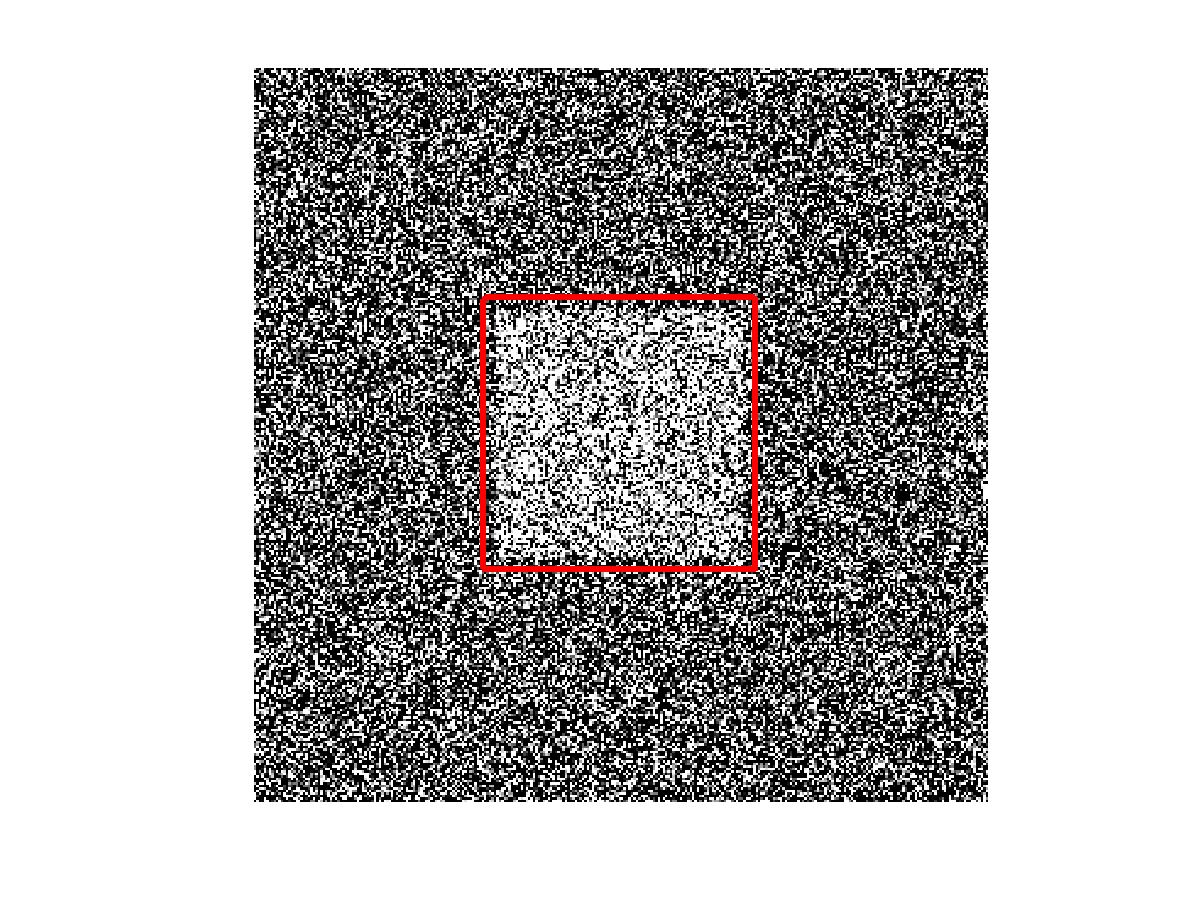}
  \includegraphics[clip,trim=125 50 110
  40,totalheight=\fHeight]{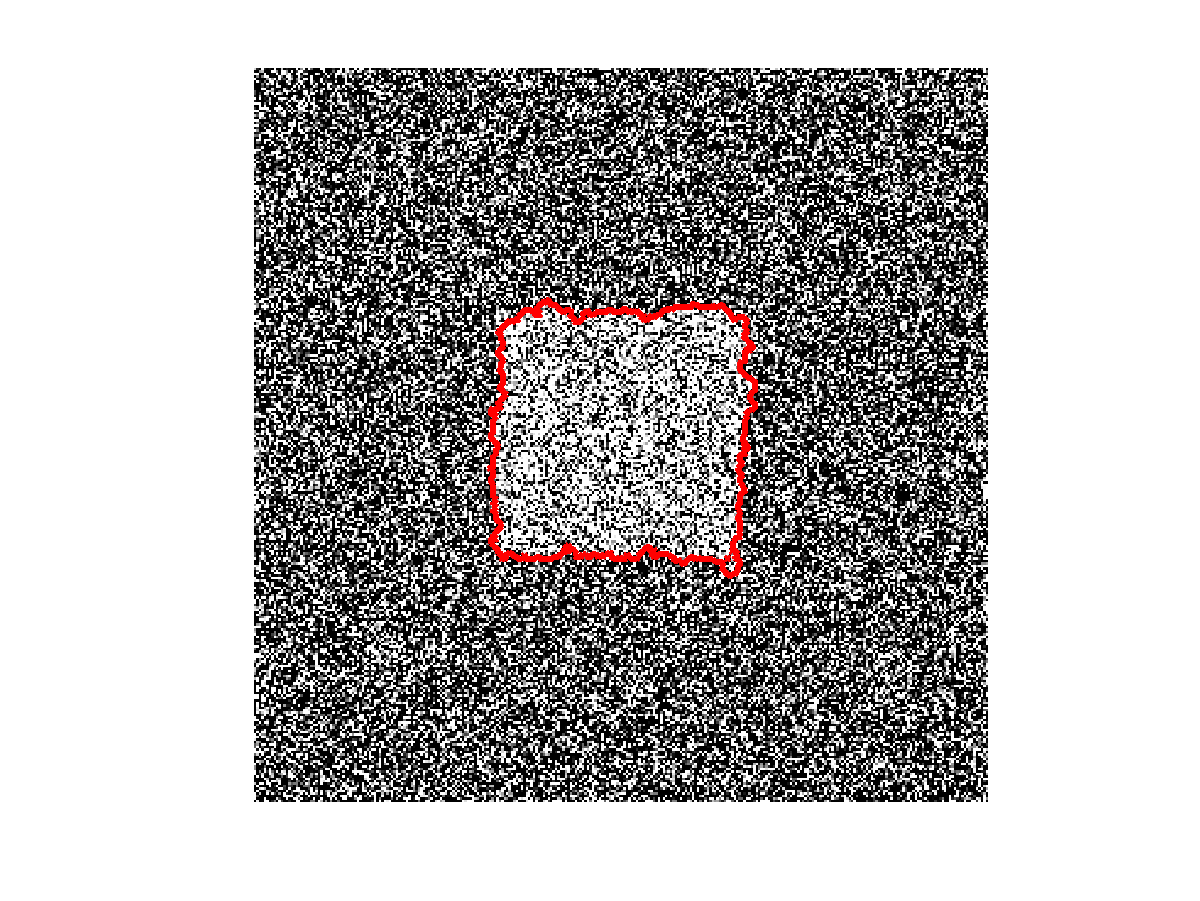}
  \caption{[Top]: Sequential coarse-to-fine methods use the result of
    segmentation (red) from the coarse scale to initialize (yellow)
    the finer scales, and may lose coarse structure of the coarse
    segmentation solution without additional heuristics. Note that the
    result of segmentation of the coarse scale is the left image in
    red (the blurred image is not shown), and towards the right
    segmentation is done at finer scales. [Bottom]: Our parallel
    coarse-to-fine approach considers a continuum of scales all at
    once and has a coarse-to-fine property. The evolution is shown
    from left to right.}
  \label{fig:parallel_ctf}
\end{figure}

\subsection{Related Work}

Scale space theory
\cite{koenderink1984structure,witkin1984scale,geusebroek2000color,koutaki2014scale}
has a long and rich history as a theory for analyzing images, and we
only provide brief highlights. The basic idea is that an image
consists of structures at various different scales (e.g., a leaf of a
tree exists at a different scale than a forest), and thus to analyze
an image without a-priori knowledge, it is necessary to consider the
image at \emph{all} scales. This is accomplished by blurring the image
at a continuum of kernel sizes. The most common kernel is a Gaussian,
which is known to be the only scale space satisfying certain axioms
such as not introducing any new features as the image is blurred
\cite{lindeberg1990scale}. Scale space has been used to analyze
structures in images (e.g.,
\cite{florack2000topological,van1994morphological,lindeberg1990scale,sironi2014multiscale}). This
has had wide ranging applications in stereo and optical
flow \cite{lucas1981iterative}, reconstruction
\cite{hummel1989reconstructions,ummenhofer2015global}, key -point
detection in wide-baseline matching \cite{lowe2004distinctive}, design
of descriptors for matching \cite{hassner2012sifts}, shape matching
\cite{bronstein2010scale}, and curve evolution
\cite{sapiro1993affine}, among others.

Gaussian scale spaces have also been used in image segmentation, most
notably in texture segmentation
\cite{galun2003texture,bresson2006multiscale,kokkinos2009texture},
which occur frequently in general images
\cite{maire2013progressive,arbelaez2014multiscale}, where the need for
scale information is cogent.  While these methods naturally capture
important scale information, they use a \emph{global} scale space
defined on the entire image, which blurs across segmentation
boundaries. Anisotropic scale spaces
\cite{perona1990scale,aujol2006structure} have been applied to reduce
blurring across boundaries, but this could blur across regions where
edges are not salient. Recently, \cite{khan2015shape} have addressed
this issue by using discrete scales computed locally within the
evolving regions of the segmentation. However, only a discrete number
of scales are used and the method does not exhibit a coarse-to-fine
property, which is the focus of this work. Such methods for
segmentation have been numerically implemented with various
optimization methods, including level sets \cite{osher1988fronts}, and
more recently convex optimization methods
\cite{pock2009algorithm}. The energy we consider is not convex, and
thus we rely on gradient descent on curves. The energy we consider
involves optimization with partial differential equation (PDE)
constraints, and thus we apply optimization techniques from
\cite{aubert2003image,delfour2011shapes}.

Coarse-to-fine methods, where coarse representations of the image or
objective function are processed and then finer aspects of the data
are successively revealed, have a long history in computer vision. One
such work is \cite{blake1987visual}. In these methods, data or the
objective function is smoothed, and the smoothed problem is
solved. The result is used to initialize the problem with less
smoothing, where finer details of the data are revealed. The hope is
that this result retains aspects of coarse solution, while gradually
finding finer detail. However, without additional heuristics such as
restricting the finer solution to be around the solution of the coarse
problem, there is no guarantee that coarse structure is preserved when
solving the finer problem. Recently, \cite{mobahi2015coarse} provided
analysis and derived closed form solutions for the smoothing of the
objective in problems of point cloud matching. Our method uses a
single energy integrating over a \emph{continuum} of scales in
parallel, and we optimize this energy directly. The optimization is
dominated by coarse aspects of the data without ignoring fine aspects
initially. Then fine aspects become more prominent, but the coarse
structure is preserved without any heuristics.

\comment{
Using an energy that incorporates
    a discrete number of scales \cite{khan2015shape} (middle) in this
    example captures the desired segmentation, however it does not
    exhibit coarse-to-fine behavior.
}

Since we apply our method to the problem of segmenting moving objects
in video based on motion, we highlight some aspects of that literature
most relevant to this work. Methods for motion segmentation are based
on optical flow (e.g., \cite{sun2010secrets}). Piecewise parametric
models for motion of regions in segmentation are used in e.g., 
\cite{wang1994representing,cremers2005motion}. Non-parametric warps
are used for motion models
(e.g.,\cite{ochs2014segmentation,sun2013fully,yang2015self}). Our goal
here is \emph{not} to estimate motion, but rather we use existing
techniques for motion estimation, and improve the segmentation of
regions by replacing the data term with our novel energy.

\section{Energy and Optimization}

In this section, our goal is to construct the energy with preference
to segmenting the coarse structure of the image(s) over finer
structure, without losing the ability to include fine-scale structure
in the segmentation. We use a scale space defined within regions of
the image, called a \emph{Shape-Tailored Scale Space}, which computes
coarse-structure of the data without blurring across region
boundaries, to construct the energy. Finally, we show that in
optimizing the energy, which requires multiple scales to define, it is
not necessary to compute multiple scales of the data, providing a
convenient implementation using the natural resolution of the image.

\subsection{Shape-Tailored Heat Scale Space}
The Gaussian Scale Space (see Figure~\ref{fig:gaussian_ss}),
constructed by smoothing the image with a Gaussian at a continuum of
scales (variances), has been shown to be the only scale space
satisfying natural axiomatic properties, which includes non-creation
of new structures in the data with increasing scale. The last property
implies that increasing scales represent coarser representations of
the data. The Gaussian Scale Space can be generalized to be defined
within regions (subsets of the image) of arbitrary shape by using the
Heat Equation. The solution to the Heat Equation defaults to Gaussian
smoothing when the domain is $\R^2$, and approximately so when the
domain is a rectangle, as in an image. The Heat Equation, defined in a
region $R$, is defined as follows:
\begin{equation}
  \label{eq:heat_eqn}
  \begin{cases}
    \partial_t u(t,x) = \Delta u(t,x) & x\in R, \, t>0 \\
    \nabla u(t,x)\cdot N = 0 & x\in \partial R, \, t>0 \\
    u(0,x) = I(x) & x\in R
  \end{cases}
\end{equation}
where $u : [0,+\infty) \times R \to \R^k$ denotes the scale space,
$R \subset \Omega \subset \R^2$ is the domain (or subset) of the image
$\Omega$, $I$ is the image, $\partial R$ denotes the boundary of $R$,
$N$ is the unit outward normal vector to $R$, $\nabla$ denotes the
vector of partials, $\Delta$ denotes the Laplacian, $\partial_t$
denotes the partial derivative with respect to $t$, and $t$ is the
scale parameter parameterizing the scale space. Note that $t$ is
related to $\sigma$, the standard deviation of the Gaussian kernel, by
$\sigma = \sqrt{2t}$, in the case the domain is $\R^2$.

\def\fHeightSS{0.7in}
\begin{figure}[bt]
  \centering
  \begin{tabular}{ccccc}
    $t=0$ & $t=50$ & $t=100$ &$t=150$ &$t=200$ \\
  \includegraphics[totalheight=\fHeightSS]{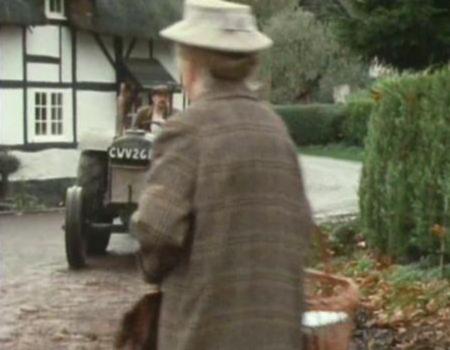} &
  \includegraphics[totalheight=\fHeightSS]{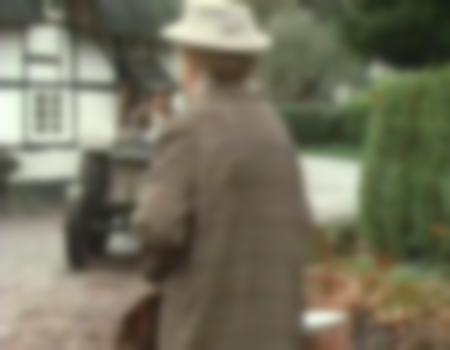} &
  \includegraphics[totalheight=\fHeightSS]{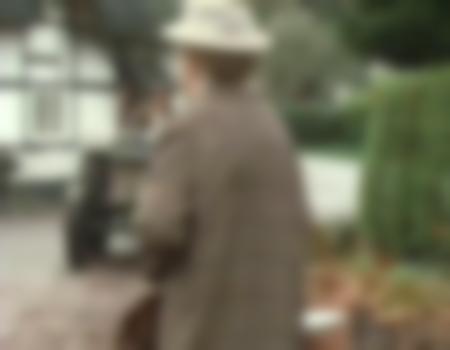} &
  \includegraphics[totalheight=\fHeightSS]{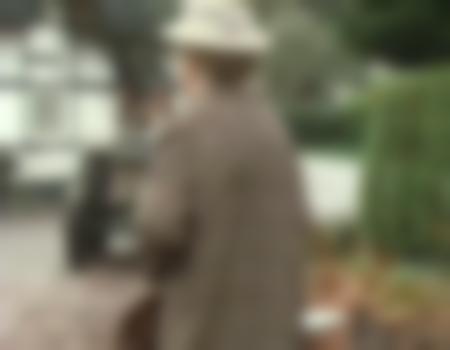} &
  \includegraphics[totalheight=\fHeightSS]{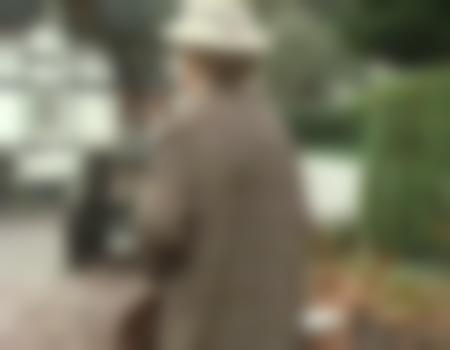}
  \end{tabular}
  \caption{Gaussian scale space (solution of Heat Equation) for
    various times (scales). Notice the quick diffusion of fine scale
    structures, and the persistence of coarse structure over much of
    the scale space. The persistence of coarse structure is important
    in defining our coarse-to-fine segmentation scheme.}
  \label{fig:gaussian_ss}
\end{figure}

This construction of the Gaussian Scale Space using the Heat Equation
is particularly useful for segmentation, as it allows us to
conveniently compute coarse representations of the data when $R$ is a
region of a segmentation. By using the PDE in \eqref{eq:heat_eqn}, we
may naturally smooth only within an arbitrary shaped subset of the
image without integrating information across the region boundary. If
the regions are chosen to be the correct segmentation, this avoids
blurring data across segmentation boundaries. However, one does not
know the segmentation a-priori, and thus the regions are
simultaneously estimated together with the scale spaces within the
regions in the optimization problem that we define next.

\subsection{Coarse-Scale Preferential Energy}
A property of the Gaussian scale space that is relevant in defining
our coarse-scale preferential energy is that the Heat Equation removes
the fine structure of the image in short time, and spends more of its
time removing coarse structure (see
Figure~\ref{fig:gaussian_ss}). Therefore, integrating a data term in
segmentation problems over the scale parameter of the Heat Equation,
would give preference to segmentations that separate the coarse
structure of the image without ignoring contributions from the fine
structure. This intuition leads us to construct the following
segmentation energy defined on possible segmentations of the domain as
a data term:
\begin{equation} \label{eq:energy}
  E( \{R_i\}_{i=1}^N ) = \sum_{i=1}^N E_i(R_i), \quad
  E_i(R_i) = \int_0^T\int_{R_i} | u_i(t,x) - a_i |^2 \ud x \ud t,
\end{equation}
where $T>0$ is the final time, $\{R_i\}_{i=1}^N$ are a collection of
regions forming the segmentation, and $a_i \in \R^k$ is the average
value of $u_t$ within $R_i$. It can be shown that $a_i$ is independent
of $t$. The parameter $T$ will be eliminated below as we take the
limit as $T\to \infty$ \footnote{Note that as $t \to \infty$, the
  solution of the Heat Equation approaches the average value of the
  input, i.e., $u_i(t,x) \to a_i$. Thus, very coarse scale components
  of the scale space are mitigated by the energy. This means fine
  aspects of the data still play a role in the energy.}. This energy
is the mean-squared error of the image within the region across all
scales. It generalizes common single scale segmentation models,
including piecewise constant Mumford-Shah (Chan-Vese
\cite{vese2002multiphase}), and piecewise smooth Mumford-Shah
\cite{mumford1989optimal}. Note that one way to generalize the latter
is by taking the initial condition for the Heat Equation
\eqref{eq:heat_eqn} as $I-f$ where $f$ is a smooth version of $I$;
this results in $a_i=0$. We will see in Section~\ref{sec:moseg} that
it also generalizes motion segmentation models.

To further justify the intuition that the energy chooses regions so
that coarse structure of the image has more influence than the fine
structure, we analyze a term of the energy above in Fourier
domain. Since Fourier analysis is convenient in rectangular domains,
we analyze the energy when the region is $R_i=\R^2$. In this case, the
energy can be written in terms of its Fourier transform as:
\begin{lemma} Suppose $I : \R^2 \to \R$ and $a = \int_{\R^2} I(x) \ud x = \int_{\R^2} u(t,x) \ud x$. Then
  \begin{equation}
    E =  \int_0^{\infty}\int_{\R^2} |u(t,x)-a|^2 \ud x\ud t = \int_{\R^2} |H(\omega)\hat I(\omega)|^2 \ud \omega, \,\, \mbox{where } H(\omega) = \frac{1}{\sqrt{2} |\omega|},
  \end{equation}
  where $\hat I$ denotes the Fourier transform, and $\omega$ denotes frequency.
\end{lemma}
The proof is a straight forward application of Parseval's Theorem,
details of which can be found in Supplementary materials. The function
$H$ decays the high frequency components of $I$ at a linear rate, thus
the energy gives preference to the coarse image structure. Note that
without integrating over the scale space, the energy in Fourier domain
has the same expression, except with $H=1$, thus having equal
preference to both coarse and fine structure. Since Fourier analysis
is not particularly simple for regions that are not rectangular, the
original energy is not formulated in Fourier domain. However, we show
in the next sub-section that the optimization of the energy using the
scale-space can be expressed without computing the entire scale space,
which is convenient in implementation.

\subsection{Constrained Optimization Problem}
\label{subsec:opt_problem}

The energy of interest \eqref{eq:energy} is a function of regions, and
thus we design an optimization scheme with respect to the
regions. Since the integrand of the energy depends on the regions
nonlinearly, as the Heat Equation is a function of the region, the
energy is not convex, and thus we apply gradient descent. In order to
compute the gradient, we formulate the energy minimization as a
constrained optimization problem. That is, we treat the minimization
of the energy \eqref{eq:energy} as defined on both the regions $R_i$
\emph{and} $u_i$ with the constraint that $u_i$ satisfies the Heat
Equation \eqref{eq:heat_eqn}. This formulation allows us to apply the
technique of Lagrange multipliers, which makes computations
simpler. In particular, the technique allows us to decouple the
nonlinear dependence of $u_i$ on $R_i$, since the latter variables can
be treated independently.

Since all terms of the energy \eqref{eq:energy} have the same
form, we focus on computing the gradient for any one term. For
convenience in notation, we avoid the subscript $i$ denoting the index
of the region. We may formulate the energy as a function of region
$R$, $u$, and a Lagrange multiplier $\lambda : [0,T]\times R \to \R^k$
with the constraint that $u$ satisfies the Heat Equation:
\begin{equation} \label{eq:energy_lagrange}
  E(R, u, \lambda) = \int_0^T \int_R f(u) \ud x \ud t +
  \int_0^T \int_R \left( \nabla \lambda \cdot \nabla u +
    \lambda \partial_t u \right) \ud x \ud t.
\end{equation}
We have excluded the dependencies on $x,t$ for convenience of
notation. We have also provided a more general form of the squared
error with a general function $f$ of $u$. The second term comes from
the weak form of the Heat Equation. Note that integrating by parts to
move the gradient from $\lambda$ to $\nabla u$ gives the classical
form of the Heat Equation in \eqref{eq:heat_eqn}. Therefore, the
second term is indeed obtained by the usual Lagrange multiplier
technique.

We may now compute the gradient for $E$ \eqref{eq:energy_lagrange} by
deriving the optimizing conditions in $u$ and $\lambda$. Optimizing in
$\lambda$ simply results in the original Heat Equation constraint, so
we compute the optimizing condition for $u$ by computing the
derivative (variation) of $E$ with respect to $u$. This results in a
solution for $\lambda$ as given below:
\begin{lemma}
  The Lagrange multiplier $\lambda$ satisfies the following Heat
  Equation with forcing term, evolving backwards in time: 
  \begin{equation}
    \begin{cases}
      \partial_t \lambda(t,x) + \Delta \lambda(t,x) = f'(u(t,x)) & x\in R\times
      [0,T] \\
      \nabla \lambda(t,x) \cdot N = 0 & x\in \partial R \times [0,T] \\
      \lambda(T,x) = 0 & x\in R
    \end{cases}.
  \end{equation}
  The solution of this equation can be expressed with Duhamel's
  Principle \cite{evans2010partial} as
  \begin{equation}
    \lambda(t,x) = -\int_t^T F(s-t, x; s) \ud s.
  \end{equation}
  where $F(\cdot,\cdot ; s) : [0,T] \times R \to \R$ is the solution
  of the forward heat equation \eqref{eq:heat_eqn} with zero forcing
  and initial condition $f'(u)$ evaluated at time $s$, i.e.,
  \begin{equation}
    \begin{cases}
      \partial_t F(t,x; s) - \Delta F(t,x; s) = 0 & x\in R\times
      [0,T] \\
      \nabla F(t,x; s) \cdot N = 0 & x\in \partial R \times [0,T] \\
      F(0,x; s) = f(s,x) & x\in R
    \end{cases}.
  \end{equation}

In the case that $f(u) =(u-a)^2$, $\lambda$ can be expressed as
\begin{equation} \label{eq:lambda_explicit}
  \lambda(t,x) = -2\int_t^T ( u(2s-t,x) - a ) \ud s.
\end{equation}

\end{lemma}
The formula for $\lambda$ in \eqref{eq:lambda_explicit} is convenient
as a numerical integration scheme is no longer required after
computation of the scale space $u$.

With the optimizing conditions for $u$ and $\lambda$ of $E$, we can
now compute the gradient of the energy $E$ with respect to $R$ in
terms of $\lambda$ and $u$:
\begin{proposition}
  The gradient of $E$ with respect to the boundary $\partial R$ can be expressed as
  \begin{equation} \label{eq:grad_energy}
    \nabla_{\partial R} E = \int_0^T
    \left[ f(u) + \nabla \lambda \cdot \nabla u +
      \lambda \partial_t u \right] \ud t \cdot N,
  \end{equation}
  where $N$ is the normal vector to $\partial R$. In the case that $f(u) = (u-a)^2$ and as $T$ gets large, the gradient approaches
  \begin{equation} \label{eq:grad_simple}
    \nabla_{\partial R} E = \left(-\frac 1 2 |\nabla \lambda(0) |^2 -
      \lambda(0)[u(0)-a]\right) N,  \quad
    \lambda(0,x) = -\int_0^{2T} (u(s,x) -a) \ud s,
  \end{equation}
  where $\lambda(0)$ and $u(0)$ denote the functions $\lambda$ and $u$ at
  time zero.
\end{proposition}
The simplification in \eqref{eq:grad_simple} of the gradient is
particularly convenient since it only involves explicit functions of
the scale space $u$. More conveniently, we may express $\lambda(0)$ as
the solution of the Poisson equation at the native scale of the image so that
it is not necessary to compute the whole scale space $u$ to evaluate the gradient
\eqref{eq:grad_simple}:
\begin{lemma}
  As $T$ gets large, $\lambda(0)$ defined in \eqref{eq:grad_simple} approaches
  the solution of the Poisson equation:
  \begin{equation} \label{eq:poisson_eqn}
    \begin{cases}
      -\Delta \lambda(0,x) = a- u(0,x) & x\in R \\
      \nabla \lambda(0,x) \cdot N =  0 & x\in \partial R\\
      \mean{\lambda(0)} = 0
    \end{cases},
  \end{equation}
  where $\mean{\lambda(0)} = \frac{1}{|R|} \int_{R} \lambda(0,x) \ud x$ is the average value over $R$, and $|R|$ is the area of $R$.
\end{lemma}
\begin{proof}
  Since the proof is short, we provide it here:
  \begin{multline*}
    \Delta \lambda(0) = -\Delta \int_0^{2T} (u(s) - a) \ud s =
    -\int_0^{2T} \Delta u(s) \ud s = - \int_0^{2T} \partial_su(s)
    \ud s \\ = -( u(2T) - u(0) ) \to -( a - u(0) ) \mbox{ as } T\to \infty.
  \end{multline*}
\end{proof}

\comment{
\begin{figure}
  \caption{The sum of the solution to the heat equation over all
    scales approaches the solution of the Poisson equation defined on
    the native scale of the image. This property simplifies
    computation by avoiding computation of the scale space.}
\end{figure}
}

In practice, taking $T$ all the way to infinity may enforce too much
of the coarse structure, and in segmenting objects with fine details,
the evolution may take too long to finally determine the fine scale
structure. Thus, rather than taking $a$ to approximate $u(2T)$ in the
above computation, we instead approximate it with the solution of the
equation:
\begin{equation} \label{eq:poisson_approx}
  \begin{cases}
    v(x)-\alpha \Delta v(x) = I(x) & x\in R \\
    \nabla v(x)\cdot N = 0 & x\in \partial R
  \end{cases},
\end{equation}
which smooths $I$ (larger $\alpha$ smooths $I$ more and can be
regarded as a maximum scale parameter). For a fixed scale, this
solution qualitatively behaves similar to the Heat Equation (they are
both low-pass filters). In experiments, this approximation still gives
the desired coarse-to-fine behavior. Solving \eqref{eq:poisson_eqn}
with right hand side $v-u(0)$ and $\alpha$ set to the maximum
desirable scale approximates $-\int_0^{2T} ( u(t,x)-a ) \ud t$ for
finite $T$. We set $\alpha=20$ in experiments. The reason for using
the equation $v-\alpha \Delta v = I$ is that it is computationally
less costly to solve than solving the Heat Equation directly, and it
allows for fast updates as the regions evolve using the solution from
the previous iteration as a warm start.

In summary, we compute the solution of the Poisson equation
\eqref{eq:poisson_eqn} with right hand side $v - u(0)$ at the native
scale of the image to solve for $\lambda(0)$ in the computation of an
approximate gradient of $E_i$ in \eqref{eq:energy}. Then, the gradient
is computed using the formula in \eqref{eq:grad_simple} that requires
only simple operations (partial derivatives, squaring, etc). Therefore,
the effects of the scale space are compressed into two equations at
the native scale of the image.

\comment{
Summarizing, to compute an approximation of the gradient of $E_i$ in
\eqref{eq:energy}, one simply computes the solution of the Poisson
equation \eqref{eq:poisson_eqn} with right hand side $v-u(0)$ at the
native scale of the image to solve for $\lambda(0)$, and then computes
the gradient using the formula in \eqref{eq:grad_simple}, which
requires only simple operations (partial derivatives, squaring,
etc). Therefore, the effects of the scale space are compressed into
two equations at the native scale of the image.
}

\subsection{Multi-label Scheme and Implementation}
We now present a method for implementing the gradient flow derived in
the previous section. Since we are interested in applications with
possibly many regions in the segmentation of the image, we present a
method for implementing the gradient flow when there are $N$
regions. To achieve sub-pixel accuracy, we use relaxed indicator
functions $\phi_i : \Omega \to [0,1]$ with $i=1,\ldots, N$ to
represent the regions. Denote by $G_i$ the quantity
\eqref{eq:grad_simple} multiplying the normal vector for region $R_i$:
\begin{equation} \label{eq:gradient_terms}
  G_i  = -\frac 1 2 |\nabla \lambda_i(0) |^2 - \lambda_i(0)[u_i(0)-a], \quad
  \mbox{in } D(R_i)
\end{equation}
where $D(R_i)$ is a small dilation of $R_i$ and $\lambda_i(0)$ is the
solution of \eqref{eq:poisson_eqn} computed in this set. The extension
beyond the region is done so that the evolution of $\phi_i$ can be
defined around the curve, as in level set methods
\cite{osher1988fronts}. We can now derive the updated scheme for the
$\phi_i$ so that the zero level set evolves in such a way that it
matches the curve evolution induced by the gradient descent. This is
given in Algorithm~\ref{alg:multilabel_scheme}.

\begin{algorithm}
\begin{algorithmic}[1]
  \State Input: An initialization of $\phi_i$
  \Repeat
  \State Compute regions:
  $R_i = \{ x\in \Omega \, : \, i = \mbox{argmax}_j \phi_j(x) \}$

  \State Dilated regions $R_i$: $D(R_i)$
  \State Compute $\lambda_i$ in $D(R_i)$ by solving the Poisson equation \eqref{eq:poisson_eqn}
  \State Compute $G_i$ for band pixels $B_i = D(R_i)\cap
  D(\Omega\backslash R_i)$
  \State Update pixels $x\in D(R_i)\cap D(R_j)$ as follows:
  \[
  \phi_i^{\tau+\Delta\tau}(x) = \phi_i^{\tau}(x) -\Delta\tau( G_i(x) -
  G_j(x) )|\nabla \phi_i^{\tau}(x)| + \varepsilon \Delta \phi_i^{\tau}(x).
  \]
  \State Update all other pixels as
  \[
  \phi_i^{\tau+\Delta\tau}(x) = \phi_i^{\tau}(x) + \varepsilon \Delta \phi_i^{\tau}(x).
  \]
  \State Clip values between 0 and 1: $\phi_i = \max\{ 0, \min\{ 1, \phi_i \}\}$.
  \Until { regions have converged  }
\end{algorithmic}
\caption{Multi-label Gradient Descent}
\label{alg:multilabel_scheme}
\end{algorithm}

The update of the $\phi_i$ in Line 7 of
Algorithm~\ref{alg:multilabel_scheme} involves the term
$\Delta \phi_i^{\tau}$, which provides smoothness of the curve. More
sophisticated regularizers (such as length regularization) may be
used, but we have found this simple regularization sufficient. We
choose $\varepsilon = 0.005$ in experiments, and this does not need
to be tuned, as it is mainly for inducing regularity for computation
of derivatives of $\phi$.

\section{Application to Motion Segmentation}
\label{sec:moseg}
In this section, we show how the results of the previous section can
be applied to motion segmentation. Motion segmentation is the problem
of segmenting objects and/or regions with similar motions computed
using multiple images of the object(s). One of the challenges of
motion segmentation is that motion is inferred typically through a
sparse set of measurements (e.g., along image edges or corners), and
thus the motion signal is typically reliable for segmentation in
sparse locations\footnote{Although motion can be hallucinated by the
  use of regularizers, this motion is unreliable in segmentation
  (typically in near constant regions).}. By using a scale space
formulation of an energy for motion segmentation, coarse
representations of the motion signal are integrated and more greatly
impact the segmentation. This property increases the reliability of
motion segmentation (Figure~\ref{fig:motion_reliability}), and the
coarse-to-fine approach captures the coarse-structure without being
impacted by fine-scale distractions at the outset.

\def\fHeightMot{0.75in}
\begin{figure}
  \centering
  \begin{tabular}{cccccc}
    residual & ctf residual & non-ctf & ctf & non-ctf & ctf \\
  \includegraphics[clip,trim=10 0 80 5, width=\fHeightMot]{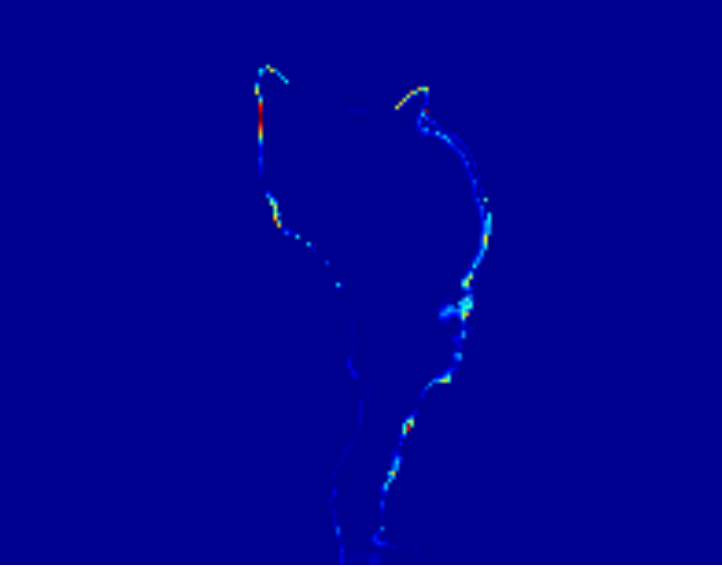} &
  \includegraphics[clip,trim=10 0 80 5, width=\fHeightMot]{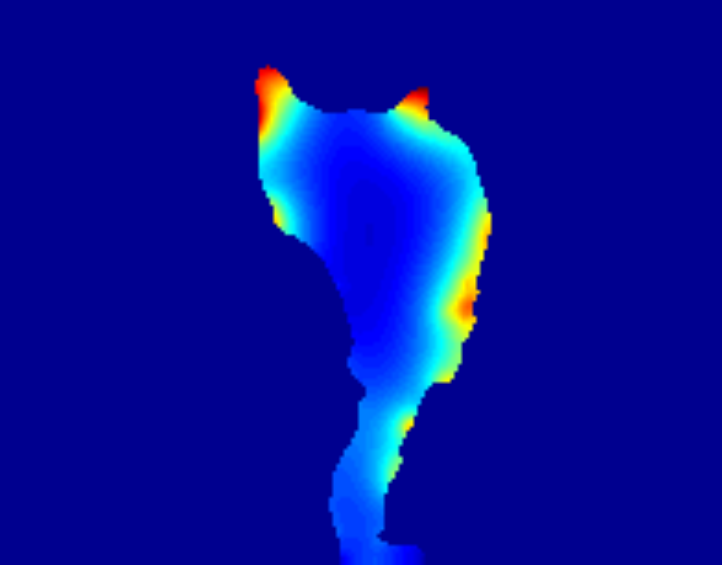} &
  \includegraphics[clip,trim=10 0 80 0, width=\fHeightMot]{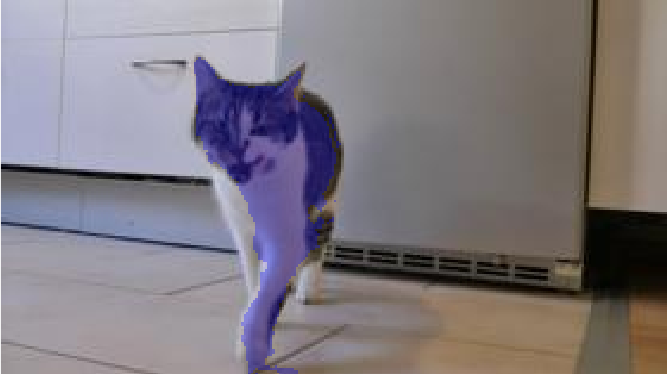} &
  \includegraphics[clip,trim=10 0 80 0, width=\fHeightMot]{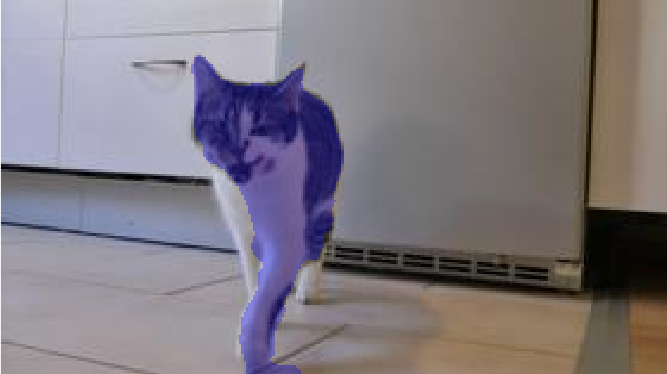} &
  \includegraphics[clip,trim=10 0 80 8, width=\fHeightMot]{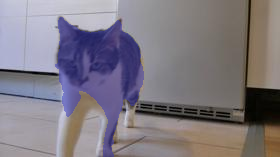} &
  \includegraphics[clip,trim=10 0 80 8, width=\fHeightMot]{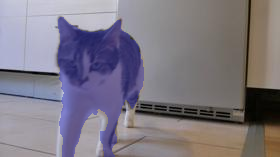}\\
  \includegraphics[clip,trim=40 50 133 20, width=\fHeightMot]{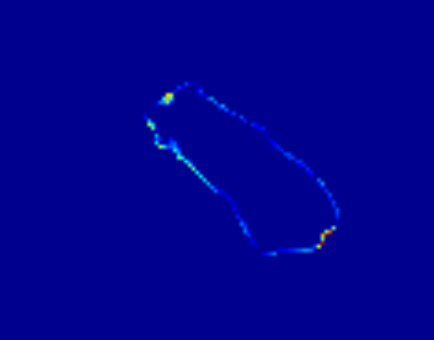} &
  \includegraphics[clip,trim=40 50 133 20, width=\fHeightMot]{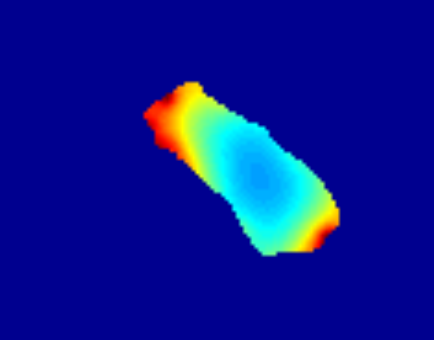} &
  \includegraphics[clip,trim=40 50 140 20, width=\fHeightMot]{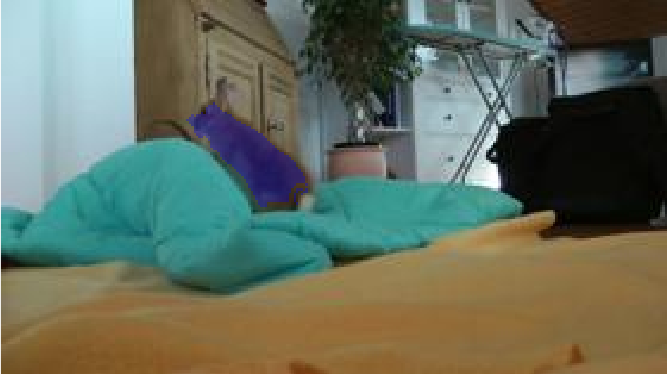} &
  \includegraphics[clip,trim=40 50 140 20, width=\fHeightMot]{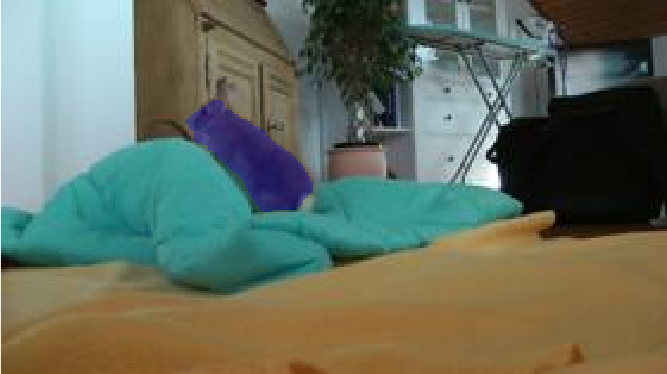} &
  \includegraphics[clip,trim=40 50 130 20, width=\fHeightMot]{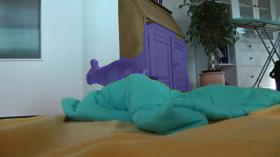} &
  \includegraphics[clip,trim=40 50 130 20, width=\fHeightMot]{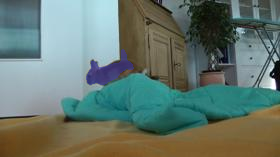}\\
    \includegraphics[clip,trim=70 40 30 10,width=\fHeightMot]{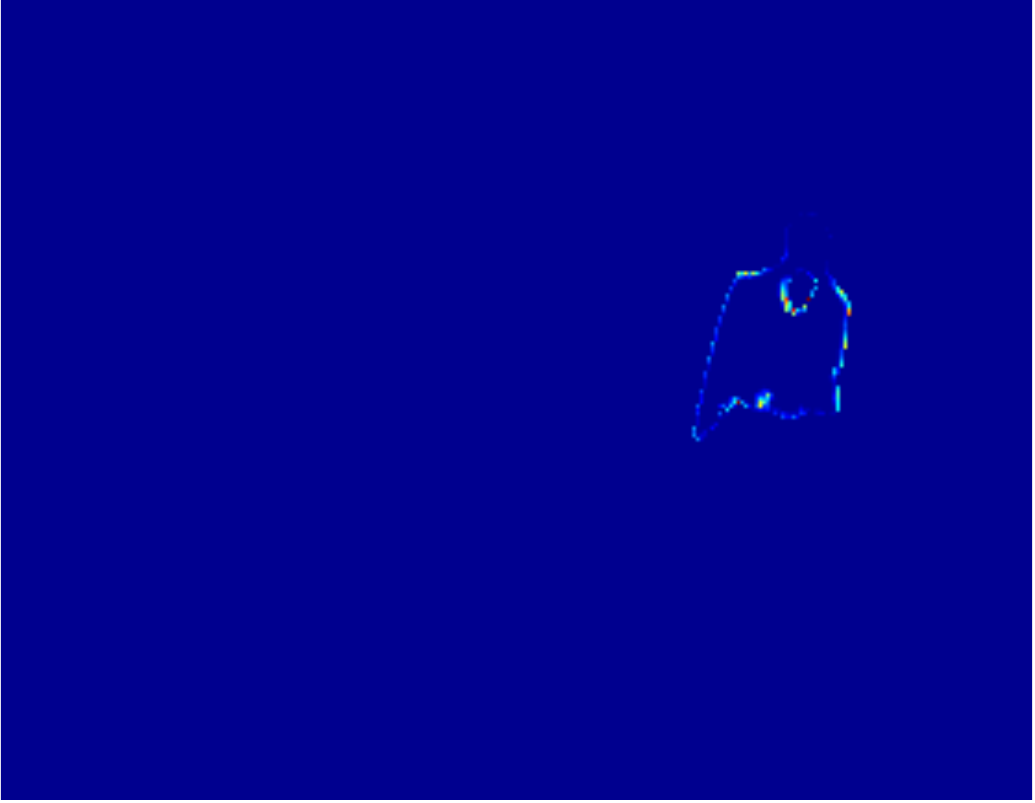} &
    \includegraphics[clip,trim=70 40 30 10,width=\fHeightMot]{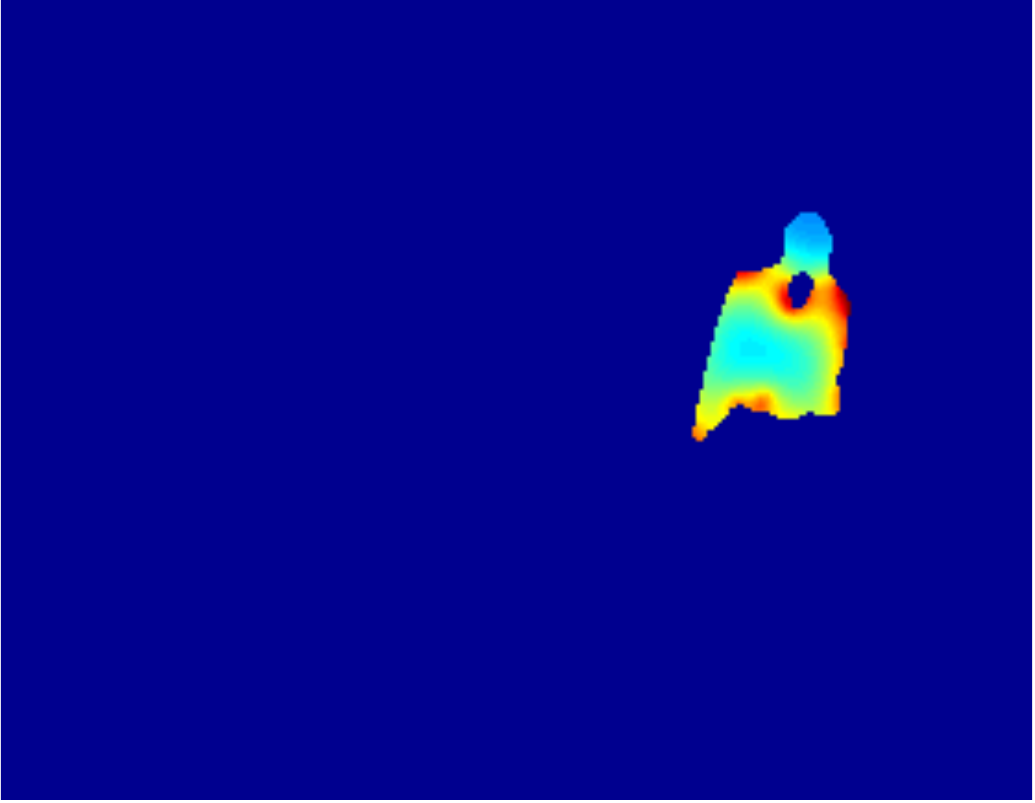} &               
    \includegraphics[clip,trim=70 40 25 10,width=\fHeightMot]{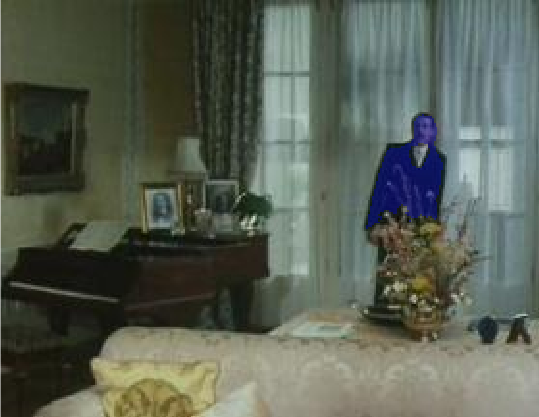} &
    \includegraphics[clip,trim=70 40 25 10,width=\fHeightMot]{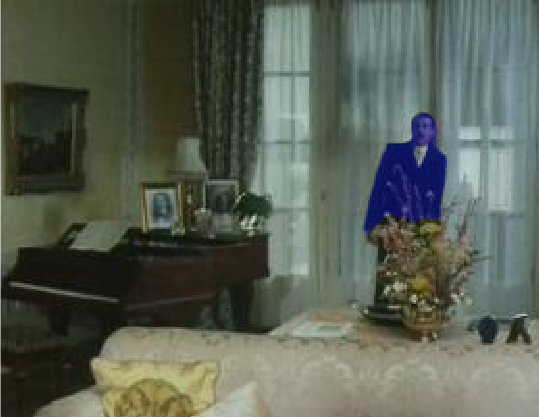} &
    \includegraphics[clip,trim=70 40 30 10,width=\fHeightMot]{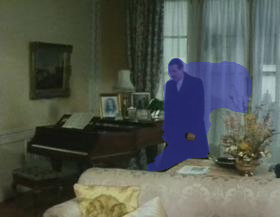} &
     \includegraphics[clip,trim=70 40 30 10,width=\fHeightMot]{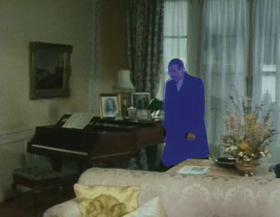}   
  \end{tabular}
  \caption{Motion residuals at a single scale are sparse (left
    column), leading to difficulties in using these cues in
    segmentation (non-ctf). Motion cues at a continuum of scales
    (ctf) provide a richer signal (2nd column), which
    increases reliability in using such cues for
    segmentation. Segmentations (in purple) are shown for a frame
    (middle two) and a few frames ahead (right two). Although errors
    in the non-ctf approach are subtle between frames, they quickly
    propagate across frames, compared to our approach.  }
  \label{fig:motion_reliability}
\end{figure}

With this motivation, we reformulate the motion segmentation problem
with scale space.  Let $I_0, I_1 : \Omega \to \R^k$ be two images of a
sequence where $\Omega$ is the domain of the image. For a given region
$R_i$, we define a mapping $w_i : R_i \to \Omega \subset \R^2$, which we
call a warp or deformation that back warps $I_1$ to $I_0$.  We assume
that $I_0$ and $I_1$ are related through $w_i$ by the Brightness
Constancy Assumption, except for occlusions, as in typical works in
the optical flow literature \cite{sun2010secrets}.  Define the pointwise
error of $w_i$ as
\begin{equation}
  \mbox{Res}_i(x) = \rho( |I_1(w_i(x)) - I_0(x)| ), \quad x\in R_i,
\end{equation}
where $\rho$ is a robust norm (for instance a truncated linear
function) to deal with the effects of deviations from Brightness
Constancy \cite{sun2010secrets}. We refer to this quantity as the
\emph{residual}. We would like to formulate an energy defined on
possible segmentations to reduce the residual and incorporate our
coarse-to-fine approach:
\begin{equation} \label{eq:Emseg}
  E_{mseg}(\{R_i\}_{i=1}^N) = \sum_{i=1}^N
  |\mean{ \mbox{Res}_i }|^2 + \int_0^T \int_{R_i \backslash O_i}
  ( u_i(t,x) - \mean{ \mbox{Res}_i } )^2 \ud x \ud t,
\end{equation}
where $O_i$ is the occluded part of $R_i$, $\mean{ \mbox{Res}_i}$ is
the mean value of the residual in $R_i \backslash O_i$, and
$u_i : [0,T] \times R_i \to \R^+$ is the scale space of the residual
in $R_i\backslash O_i$. We subtract the mean value of the residual
from the scale space so that it fits in the form of \eqref{eq:energy},
and since we would like to reduce the overall residual, we add the
mean value of the residual outside the integrand so that it is
minimized as well.  Had we not subtracted the mean value and not
integrated over scale and instead used only the native scale, this
would be the usual robust formulation of motion segmentation (e.g.,
\cite{sun2013fully}).

Because of the ambiguity of computing motion in occluded and
textureless regions, additional terms must be added to the energy, for
instance, terms involving fidelity to local appearance histograms. We
follow the formulation in \cite{yang2015self} to account for these
ambiguities, which uses an additional term with fidelity to color
histograms in the segmentation energy. Based on the reliability of the
motion residuals, the approach switches between segmentation by
residuals and color histograms.  We simply replace the classical
motion term (at a single scale) with our term \eqref{eq:Emseg}. We do
not integrate the color histogram energy over scales, as our goal is
to rely less on this term by improving the motion reliability. The
optimization involves iterative updates of the warps, occlusion, and
the regions. The technique we introduced only affects the updates of
the regions, replacing the gradient of the usual single scale motion
residual with the gradient of \eqref{eq:Emseg} computed by
\eqref{eq:gradient_terms}, and implemented using
Algorithm~\ref{alg:multilabel_scheme}. We apply our method
frame-by-frame segmenting a frame using one frame ahead and one frame
behind (so that backward motion is also used in segmentation). Then we
propagate the result to the next frame via the computed motion to
warm-start the segmentation in the next frame.

\section{Experiments}

\begin{table}[bt]
 \centering
  {
    \footnotesize
    \begin{tabular}{lc@{\hspace{1em}}c@{\hspace{1em}}c@{\hspace{1em}}c@{\hspace{1em}}c@{\hspace{1em}}c@{\hspace{1em}}c@{\hspace{1em}}c}
      & \multicolumn{4}{c}{Training set (29 sequences)} &
          \multicolumn{4}{c}{Test set (30 sequences)}\\\hline
    & P & R & F & $N/65$ & P & R & F & $N/69$ \\\hline
    \cite{grundmann2010efficient} & 79.17 & 47.55 & 59.42 & 4 & 77.11
                & 42.99 & 55.20 & 5 \\
    \cite{ochs2014segmentation} & 81.50 & 63.23 & 71.21 & 16 & 74.91 &
                                                 60.14 & 66.72 & 20  \\
    \cite{taylor2015causal} & 85.00 & 67.99 & 75.55 & 21 & 82.37 &
              58.37 & 68.32 & 17 \\
    \cite{taylor2015causal}+backward & 83.00 & 70.10 & 76.01 & 23 & 77.94 &
               59.14 & 67.25 & 15 \\
      \cite{brox2015}-Mce S(8), p(.6) &86.91 & 71.33 & 78.35 &25  &
                                                                    87.57
                & {\bf 70.19} & {\bf 77.92} & 25\\
      \cite{brox2015}-Mce S(4), p(.5) & 86.79 & {\bf 73.36} & 79.51 &28  & 86.81 & 67.96& 76.24& 25\\
      \cite{brox2015}-Mce D(4), p(.5) &85.31 & 68.70 & 76.11 &24  &
                                                                    85.95 & 65.07& 74.07& 23\\
      non-ctf \cite{yang2015self} & 89.53 & 70.74 & 79.03 &
                                                    26
            & 91.47 & 64.75 & 75.82  & 27 \\
      ctf (ours) & {\bf 93.04} & 72.68 & {\bf 81.61} &
                                                 {\bf 29}
            & {\bf 95.94} &  65.54 & 77.87 & {\bf 28}
\end{tabular}
}

\caption{FBMS-59 results. Average precision (P), recall (R),
  F-measure (F), and number of objects detected (N) over all sequences
  in training and test datasets. Higher values indicate
  superior performance. \cite{taylor2015causal},
  \cite{yang2015self} and our method are frame-to-frame methods; other
  methods process the video in batch. All methods are fully
  automatic.
}
  \label{tab:FBMS59_quant}
\end{table}

\begin{figure}
  \centering
  \footnotesize
  Frames for Increasing Time $\rightarrow$\\
  \begin{tabular}{c@{}c@{}c@{}}
    \rotatebox{90}{non-ctf} & \hspace{3pt} &
    \includegraphics[width=.15\linewidth,height=.09\linewidth]{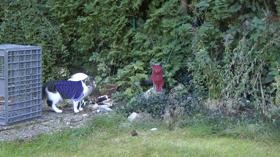}
  \includegraphics[width=.15\linewidth,height=.09\linewidth]{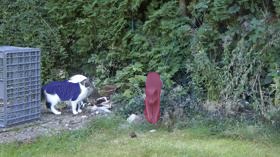}
  \includegraphics[width=.15\linewidth,height=.09\linewidth]{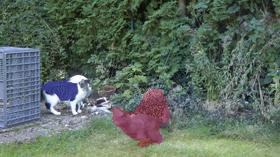}
  \includegraphics[width=.15\linewidth,height=.09\linewidth]{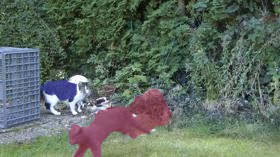}
  \includegraphics[width=.15\linewidth,height=.09\linewidth]{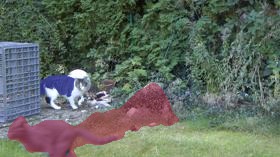}
  \includegraphics[width=.15\linewidth,height=.09\linewidth]{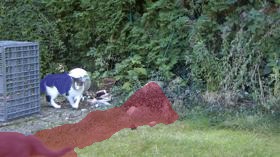}\\[-\dp\strutbox]
    \rotatebox{90}{\quad ours} & \hspace{3pt} &
    \includegraphics[width=.15\linewidth,height=.09\linewidth]{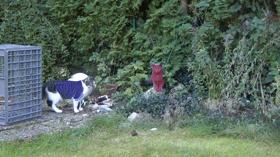}
  \includegraphics[width=.15\linewidth,height=.09\linewidth]{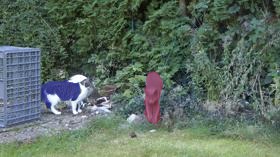}
  \includegraphics[width=.15\linewidth,height=.09\linewidth]{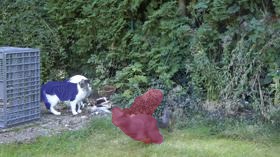}
  \includegraphics[width=.15\linewidth,height=.09\linewidth]{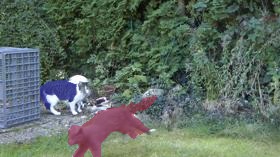}
  \includegraphics[width=.15\linewidth,height=.09\linewidth]{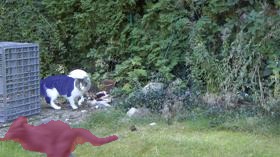}
  \includegraphics[width=.15\linewidth,height=.09\linewidth]{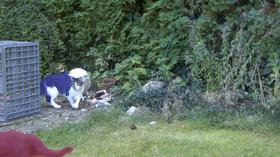}\\

    \rotatebox{90}{non-ctf} & \hspace{3pt} &
\includegraphics[width=.15\linewidth,height=.09\linewidth]{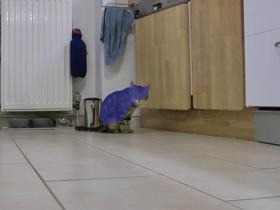}
\includegraphics[width=.15\linewidth,height=.09\linewidth]{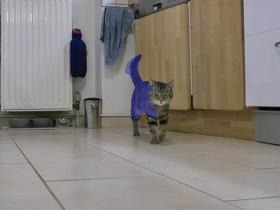}
\includegraphics[width=.15\linewidth,height=.09\linewidth]{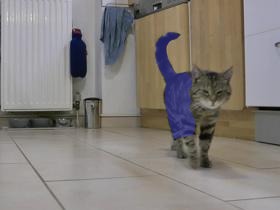}
\includegraphics[width=.15\linewidth,height=.09\linewidth]{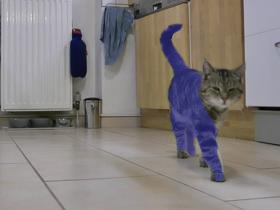}
\includegraphics[width=.15\linewidth,height=.09\linewidth]{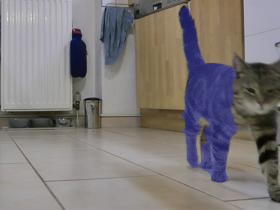}
\includegraphics[width=.15\linewidth,height=.09\linewidth]{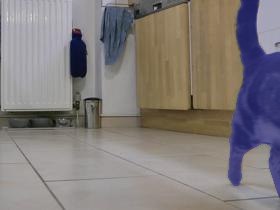}\\[-\dp\strutbox]
\rotatebox{90}{\quad ours} & \hspace{3pt} &
\includegraphics[width=.15\linewidth,height=.09\linewidth]{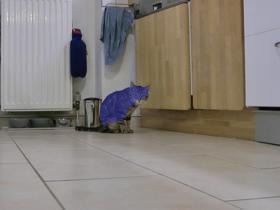}
\includegraphics[width=.15\linewidth,height=.09\linewidth]{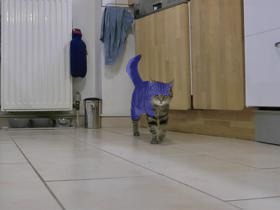}
\includegraphics[width=.15\linewidth,height=.09\linewidth]{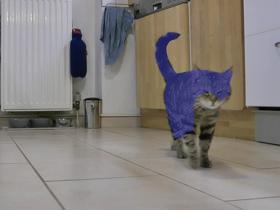}
\includegraphics[width=.15\linewidth,height=.09\linewidth]{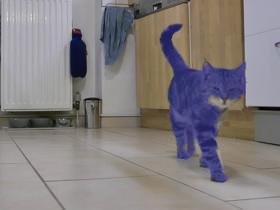}
\includegraphics[width=.15\linewidth,height=.09\linewidth]{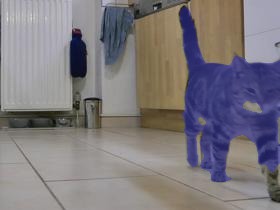}
\includegraphics[width=.15\linewidth,height=.09\linewidth]{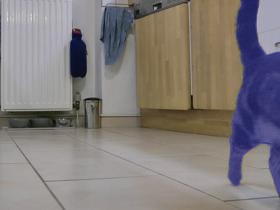}\\

\rotatebox{90}{non-ctf} & \hspace{3pt} &
\includegraphics[width=.15\linewidth,height=.09\linewidth]{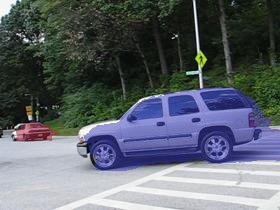}
\includegraphics[width=.15\linewidth,height=.09\linewidth]{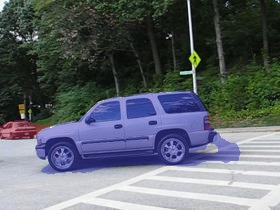}
\includegraphics[width=.15\linewidth,height=.09\linewidth]{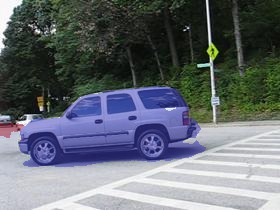}
\includegraphics[width=.15\linewidth,height=.09\linewidth]{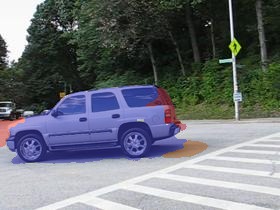}
\includegraphics[width=.15\linewidth,height=.09\linewidth]{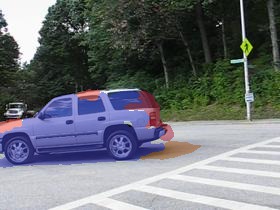}
\includegraphics[width=.15\linewidth,height=.09\linewidth]{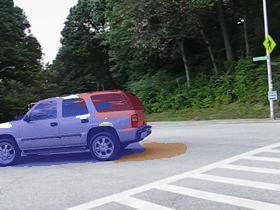}\\[-\dp\strutbox]
\rotatebox{90}{\quad ours} & \hspace{3pt} &
\includegraphics[width=.15\linewidth,height=.09\linewidth]{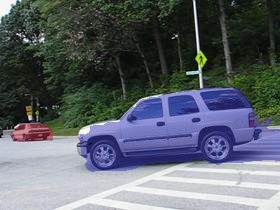}
\includegraphics[width=.15\linewidth,height=.09\linewidth]{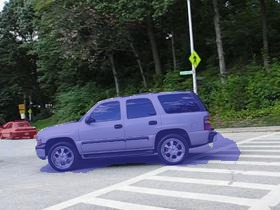}
\includegraphics[width=.15\linewidth,height=.09\linewidth]{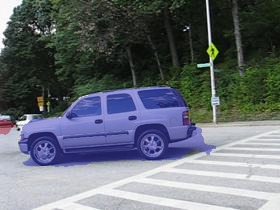}
\includegraphics[width=.15\linewidth,height=.09\linewidth]{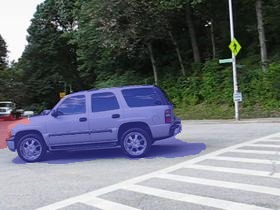}
\includegraphics[width=.15\linewidth,height=.09\linewidth]{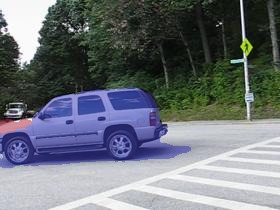}
\includegraphics[width=.15\linewidth,height=.09\linewidth]{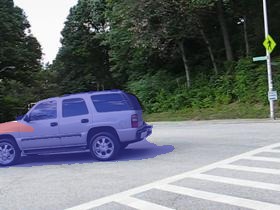}\\

\rotatebox{90}{non-ctf} & \hspace{3pt} &
\includegraphics[width=.15\linewidth,height=.09\linewidth]{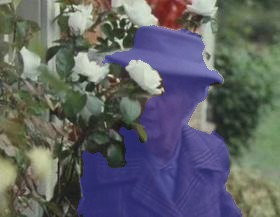}
\includegraphics[width=.15\linewidth,height=.09\linewidth]{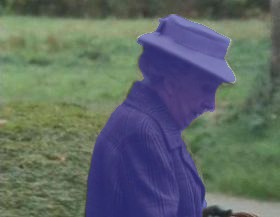}
\includegraphics[width=.15\linewidth,height=.09\linewidth]{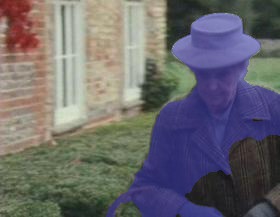}
\includegraphics[width=.15\linewidth,height=.09\linewidth]{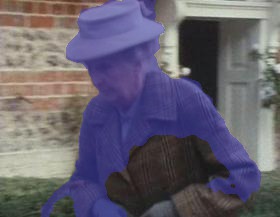}
\includegraphics[width=.15\linewidth,height=.09\linewidth]{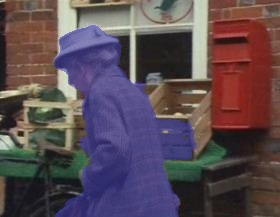}
\includegraphics[width=.15\linewidth,height=.09\linewidth]{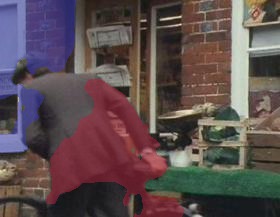}\\[-\dp\strutbox]
\rotatebox{90}{\quad ours} & \hspace{3pt} &
\includegraphics[width=.15\linewidth,height=.09\linewidth]{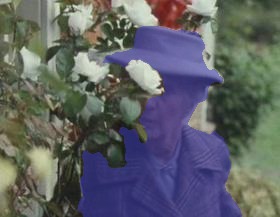}
\includegraphics[width=.15\linewidth,height=.09\linewidth]{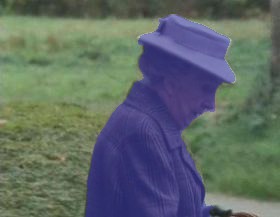}
\includegraphics[width=.15\linewidth,height=.09\linewidth]{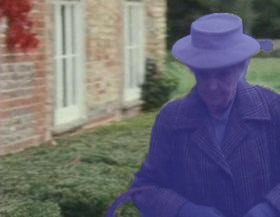}
\includegraphics[width=.15\linewidth,height=.09\linewidth]{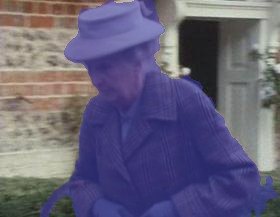}
\includegraphics[width=.15\linewidth,height=.09\linewidth]{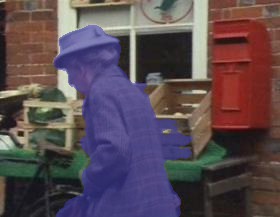}
\includegraphics[width=.15\linewidth,height=.09\linewidth]{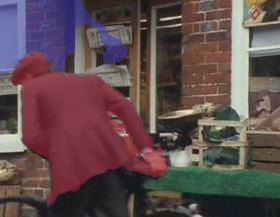}\\

\rotatebox{90}{non-ctf} & \hspace{3pt} &
\includegraphics[width=.15\linewidth,height=.09\linewidth]{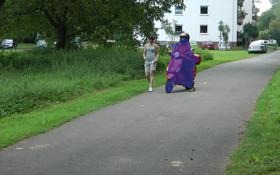}
\includegraphics[width=.15\linewidth,height=.09\linewidth]{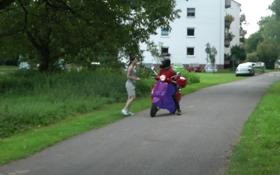}
\includegraphics[width=.15\linewidth,height=.09\linewidth]{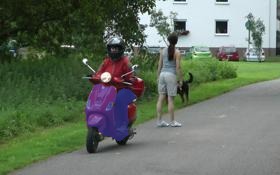}
\includegraphics[width=.15\linewidth,height=.09\linewidth]{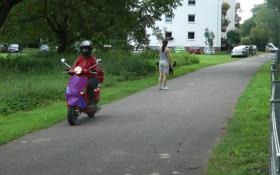}
\includegraphics[width=.15\linewidth,height=.09\linewidth]{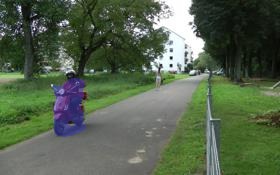}
\includegraphics[width=.15\linewidth,height=.09\linewidth]{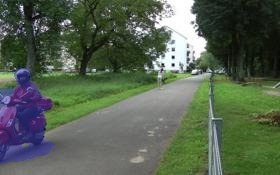}\\[-\dp\strutbox]
\rotatebox{90}{\quad ours} & \hspace{3pt} &
\includegraphics[width=.15\linewidth,height=.09\linewidth]{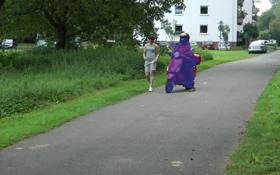}
\includegraphics[width=.15\linewidth,height=.09\linewidth]{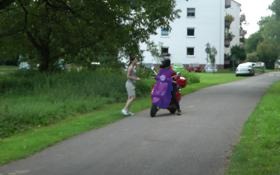}
\includegraphics[width=.15\linewidth,height=.09\linewidth]{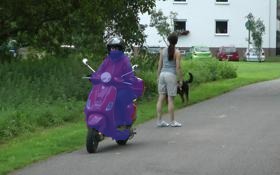}
\includegraphics[width=.15\linewidth,height=.09\linewidth]{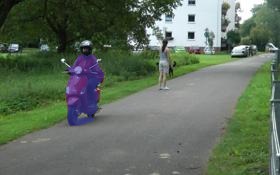}
\includegraphics[width=.15\linewidth,height=.09\linewidth]{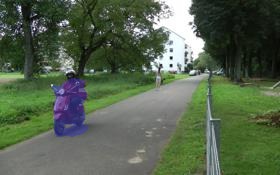}
\includegraphics[width=.15\linewidth,height=.09\linewidth]{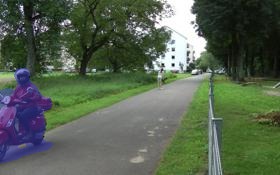}\\

\rotatebox{90}{non-ctf} & \hspace{3pt} &
\includegraphics[width=.15\linewidth,height=.09\linewidth]{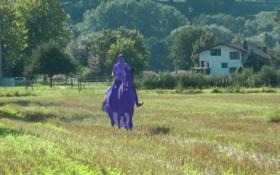}
\includegraphics[width=.15\linewidth,height=.09\linewidth]{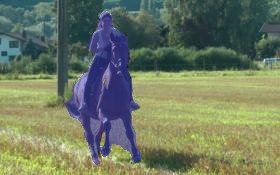}
\includegraphics[width=.15\linewidth,height=.09\linewidth]{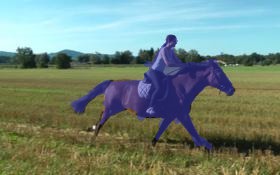}
\includegraphics[width=.15\linewidth,height=.09\linewidth]{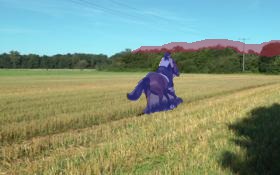}
\includegraphics[width=.15\linewidth,height=.09\linewidth]{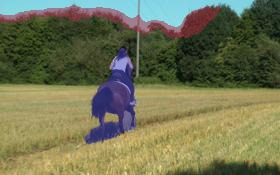}
\includegraphics[width=.15\linewidth,height=.09\linewidth]{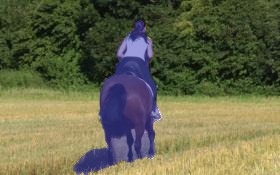}\\[-\dp\strutbox]
\rotatebox{90}{\quad ours} & \hspace{3pt} &
\includegraphics[width=.15\linewidth,height=.09\linewidth]{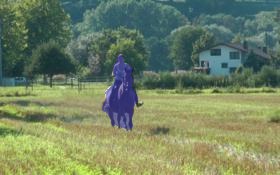}
\includegraphics[width=.15\linewidth,height=.09\linewidth]{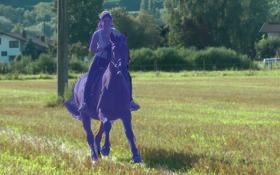}
\includegraphics[width=.15\linewidth,height=.09\linewidth]{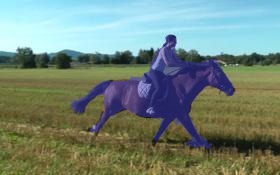}
\includegraphics[width=.15\linewidth,height=.09\linewidth]{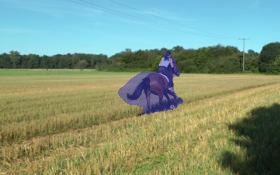}
\includegraphics[width=.15\linewidth,height=.09\linewidth]{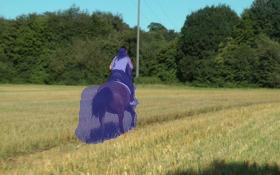}
\includegraphics[width=.15\linewidth,height=.09\linewidth]{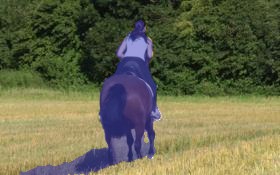}\\

\rotatebox{90}{non-ctf} & \hspace{3pt} &
\includegraphics[width=.15\linewidth,height=.09\linewidth]{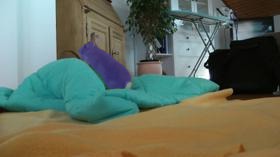}
\includegraphics[width=.15\linewidth,height=.09\linewidth]{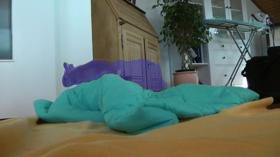}
\includegraphics[width=.15\linewidth,height=.09\linewidth]{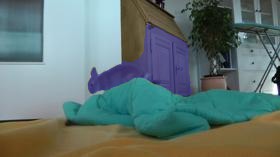}
\includegraphics[width=.15\linewidth,height=.09\linewidth]{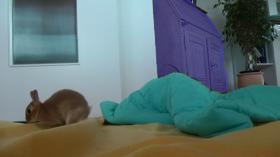}
\includegraphics[width=.15\linewidth,height=.09\linewidth]{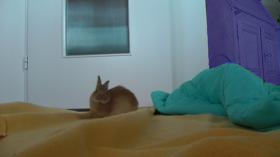}
\includegraphics[width=.15\linewidth,height=.09\linewidth]{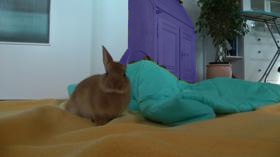}\\[-\dp\strutbox]
\rotatebox{90}{\quad ours} & \hspace{3pt} &
\includegraphics[width=.15\linewidth,height=.09\linewidth]{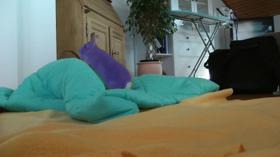}
\includegraphics[width=.15\linewidth,height=.09\linewidth]{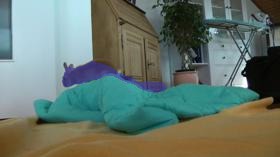}
\includegraphics[width=.15\linewidth,height=.09\linewidth]{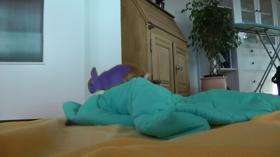}
\includegraphics[width=.15\linewidth,height=.09\linewidth]{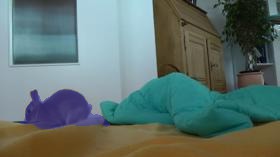}
\includegraphics[width=.15\linewidth,height=.09\linewidth]{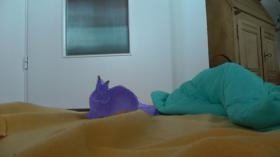}
\includegraphics[width=.15\linewidth,height=.09\linewidth]{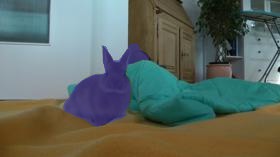}
  \end{tabular}
  \caption{Sample visual results on representative sequences for the
    FBMS-59 dataset (segmented objects in purple and red). The change
    of energy to integrate over all scales (our approach) is generally
    less sensitive to clutter than using an energy that contains only
    one scale (non-ctf).}
 \label{fig:FBMS59_qual}
\end{figure}

We test our method on a recent benchmark dataset that contains moving
objects: the Freiburg-Berkeley Motion Segmentation (FBMS-59)
\cite{ochs2014segmentation} dataset. FBMS-59 consists of two sets -
training, 29 sequences, and test, 30 sequences. Videos range between
19 and 800 frames, and have multiple objects.

\comment{
and the SegTrack v2

SegTrack v2 consists of 14 sequences ranging from 29-279
frames with multiple objects.
}

{\bf Evaluation}: FBMS-59 measures the accuracy versus ground truth of
a subset of frames, ranging from 3 to 41 for each video. The
segmentation is measured in terms of region metrics that are measured
using precision, recall, $F$ - measure, and the number of objects with
$F \geq 0.75$.

\comment{
SegTrack v2
evaluates measures the accuracy on all frames. Accuracy is measured in
terms of the average intersection over union overlap.
}

{\bf Comparison}: To demonstrate the advantage of our coarse-to-fine
energy over a corresponding single scale energy, we compare to
\cite{yang2015self}. Our approach replaces the single scale motion
term in \cite{yang2015self} with the coarse-to-fine energy described
in the previous section. Since we test on benchmarks, we also
compare to other state-of-the-art approaches, although our main
purpose is to show the improvements that occur by merely using our
coarse-to-fine energy.

{\bf Parameters}: We use the parameters in the method
\cite{yang2015self} provided in their online code, which are chosen
constant across all sequences and datasets. Our method requires one
parameter $\alpha$ in \eqref{eq:poisson_approx}. We choose it to be
$\alpha=20$ by selecting it based on a few sequences from the training
set. The parameters of \cite{yang2015self} are reported to be chosen
by using the training set of FBMS-59.

{\bf Results on FBMS-59}: Figure~\ref{fig:FBMS59_qual} shows some
representative visual results of our method and \cite{yang2015self}.
Table~\ref{tab:FBMS59_quant} shows quantitative results of the two
approaches, as well as other state-of-the-art methods. Note our method
is initialized as in \cite{yang2015self} with a clustering of
optical flow over several frames.  Visual results show our approach
generally avoids distracting clutter and thus prevents leakages in
comparison to \cite{yang2015self}. In many cases, it also captures
more of the object. The latter can be attributed to the fact that our
approach makes the motion signal more reliable, and thus it is used
more frequently than \cite{yang2015self}, which switches to color
histogram segmentation when motion cues are unreliable. Quantitative
results show that we improve the F-measure of \cite{yang2015self} by
about $2\%$ on both training and test sets, and that we increase the
number of objects detected. We also have highest F-measure on the
training set of all competing methods. On the test set, we are
out-performed by \cite{brox2015} by a slim margin of $0.05\%$
F-measure, but we do detect more objects. Note however,
\cite{brox2015} and our method are not directly comparable since in
\cite{brox2015} the video is processed in batch, whereas our method
processes the video frame-by-frame.

{\bf Computational cost}: The processing required for our approach is
small compared with the overall cost of \cite{yang2015self}.  Our
approach requires the solution of a linear equation
\eqref{eq:poisson_eqn} at each update of the regions, but the solution
does not change much between updates and so the solution from the
previous iteration provides a good initialization to the current
iteration. Our approach adds about 5 secs per frame to the total time
on average of about 30 secs per frame by \cite{yang2015self} on a
12-core processor.

\section{Conclusion}
We have presented a general energy that reformulates conventional data
terms in segmentation problems. This novel energy incorporates scale
space and consists of two important properties: scales spaces are
defined on regions, so that structures in different segments are not
blurred across nor displaced, and it exhibits a coarse-to-fine
property. The latter favors that the coarse structure of the desired
segmentation is obtained while finer structure becomes successively
obtained, without having to rely on heuristics to maintain that finer
solutions remain close to the coarse solution. Our method is based on
the Heat Equation defined on regions, and this equation was
demonstrated to have the desired properties. We
have shown application to the problem of motion segmentation, where
relying on data from a single scale is often unreliable. In this case,
other information such as color histograms must be used. However, it
is often difficult to obtain objects with complex appearance by using
color histograms, thus improving reliability of the motion signal
segments objects. Experiments on benchmark datasets have shown that
our technique improves an existing segmentation method merely by
replacing the data term of the motion residual at a single scale with
an integration over a continuum of scales. In particular, we observe
that our technique is less sensitive to clutter, and many times
increases the recall by relying more on motion cues.

\comment{Interestingly, the scale
space need not be computed to approximately optimize the energy.}

\appendix

\section{Proofs of Lemmas and Propositions}

\setcounter{lemma}{0}
\setcounter{proposition}{0}

\begin{lemma} Suppose $I : \R^2 \to \R$ and $a = \mean{I} = \mean{ u(t,\cdot) }$. Then 
  \begin{equation}
    E =  \int_0^{\infty}\int_{\R^2} |u(t,x)-a|^2 \ud x\ud t =
    \int_{\R^2} |H(\omega)\hat I(\omega)|^2 \ud \omega, \,\,
    \mbox{where } H(\omega) = \frac{1}{\sqrt{2} |\omega|}, 
  \end{equation}
  where $\hat I$ denotes the Fourier transform, and $\omega$ denotes frequency.
\end{lemma}
\begin{proof}
  Taking the Fourier transform of the Heat Equation:
  \[
  \begin{cases}
    \partial_t u(t,x) = \Delta u(t,x) & x \in \R^2  \\
    u(0,x) = I(x) & t=0
  \end{cases}
  \]
  yields:
  \[
  \partial_t \hat u(t,\omega) = (i\omega)\cdot (i\omega) \hat u(t,\omega) = 
  -|\omega|^2 \hat u(t,\omega),
  \]
  where $\hat u(t,\omega)$ is the Fourier transform of $u$. Solving this differential equation yields
  \[
  \hat u(t,\omega) = e^{-|\omega|^2 t}\hat I(\omega).
  \]
  We note that $a = 0$ when $I \in \mathbb{L}^2$ since $\hat I(0) = \int_{\R^2} I(x) \ud x$ is finite. Then by Parseval's Theorem,
  \begin{align}
    E &= \int_0^{\infty} \int_{\R^2} |u(t,x)-a|^2 \ud x \ud t = 
        \int_0^{\infty} \int_{\R^2} |\hat u(t,\omega)|^2 \ud \omega \ud t \\
      &= \int_{\R^2} \int_0^{\infty} e^{-2|\omega|^2 t} \ud t \cdot |\hat
        I(\omega)|^2 \ud \omega \\
      &= \int_{\R^2 } 
        \left. -\frac{1}{2|\omega|^2} e^{-2|\omega|^2 t}
        \right|_{t=0}^{t=+\infty} |\hat I(\omega)|^2 \ud \omega = 
        \int_{\R^2 } \frac{1}{2|\omega|^2} |\hat
        I(\omega)|^2 \ud \omega \\
    &= \int_{\R^2 } |A(\omega) \hat I(\omega)|^2 \ud \omega
  \end{align}
  where $A(\omega) = 1/(\sqrt{2} |\omega|)$.
\end{proof}

\begin{lemma}
  The Lagrange multiplier $\lambda$ satisfies the following Heat
  Equation with forcing term, evolving backwards in time: 
  \begin{equation}
    \begin{cases}
      \partial_t \lambda(t,x) + \Delta \lambda(t,x) = f'(u(t,x)) & x\in R\times
      [0,T] \\
      \nabla \lambda(t,x) \cdot N = 0 & x\in \partial R \times [0,T] \\
      \lambda(T,x) = 0 & x\in R
    \end{cases}.
  \end{equation}
  The solution of this equation can be expressed with Duhamel's
  Principle \cite{evans2010partial} as
  \begin{equation}
    \lambda(t,x) = -\int_t^T F(s-t, x; s) \ud s.
  \end{equation}
  where $F(\cdot,\cdot ; s) : [0,T] \times R \to \R$ is the solution
  of the forward heat equation Eqn.~(1) (in the paper) with zero forcing
  and initial condition $f'(u)$ evaluated at time $s$, i.e.,
  \begin{equation}
    \begin{cases}
      \partial_t F(t,x; s) - \Delta F(t,x; s) = 0 & (t,x) \in [0,T] \times R\\
      \nabla F(t,x; s) \cdot N = 0 & x\in [0,T] \times \partial R\\
      F(0,x; s) = f(u(s,x)) & x\in R
    \end{cases}.
  \end{equation}

In the case that $f(u) =(u-a)^2$, $\lambda$ can be expressed as
\begin{equation} \label{eq:lambda_explicit_2}
  \lambda(t,x) = -2\int_t^T ( u(2s-t,x) - a ) \ud s.
\end{equation}

\end{lemma}

\begin{proof}
  We define
  \begin{equation} \label{eq:energy_lagrange_2}
    E(R, u, \lambda) = \int_R \int_0^T f(u) \ud x \ud t +
    \int_R \int_0^T \left( \nabla \lambda \cdot \nabla u +
      \lambda \cdot u_t \right) \ud x \ud t.
  \end{equation}
  Integrating by parts, we have that
  \begin{align}
    E &= \int_{R\times [0,T] } \left[ f(u) - (\partial_t \lambda + \Delta
        \lambda) u \right] \ud x \ud t \\
      &+ \int_R \left. \lambda u \right|_{t=0}^{t=T} \ud x +
        \int_{\partial R \times [0,T] } 
        (\nabla \lambda \cdot N)u
        \ud s(x) \ud t,
  \end{align}
  where $\ud s$ denotes the arc-length measure of $\partial R$, and $N$ is the unit outward normal of $\partial R$. Differentiating $E$ in the direction (perturbation) $\tilde u$ of $u$ evaluated at $u$ yields
  \begin{align}
    \ud E(u) \cdot \tilde u  &= 
     \int_{R\times [0,T] } \left[ f'(u) -(\partial_t \lambda + \Delta
     \lambda)  \right] \tilde u\ud x \ud t  \\
    &+ \int_R \left. \lambda \tilde u \right|_{t=0}^{t=T} \ud x + 
      \int_{\partial R \times [0,T] } 
      (\nabla \lambda \cdot N)\tilde u
      \ud s(x) \ud t.
  \end{align}
  Note that $\tilde u (0) = 0$ since $u(0) = I$ is fixed and thus may not be perturbed. We may choose $\nabla \lambda \cdot N = 0$ on $\partial R$ and $\lambda(T) = 0$. We are interested in $u$ such that $\ud E(u) \cdot \tilde u = 0$ for all $\tilde u$. This yields the condition that 
  \begin{equation}
    \begin{cases}
      \partial_t \lambda(t,x) + \Delta \lambda(t,x) = f'(u(t,x)) & x\in R\times
      [0,T] \\
      \nabla \lambda(t,x) \cdot N = 0 & x\in \partial R \times [0,T] \\
      \lambda(T,x) = 0 & x\in R
    \end{cases}.
  \end{equation}
  To express the solution to the above equation in a more convenient form, we may use Duhamel's Principle. The latter states that a linear PDE with forcing term $\delta(t-s)$ is equivalent to the same PDE with zero forcing and initial condition at $s$ of $1$. We may express the forcing term as $\int_{ [0,T] } f(u(x,s)) \delta(s-t)\ud s $, and thus combining linearity of the PDE with Duhamel's Principle yields that 
  \begin{equation} \label{eq:lambda}
    \lambda(t,x) = -\int_t^T F(s-t, x; s) \ud s,
  \end{equation}
  i.e., it is the sum of solutions of the PDE with zero forcing and initial condition $f(u(x,s))$ at time $s$, specifically,
  \begin{equation}
    \begin{cases}
      \partial_t F(t,x; s) - \Delta F(t,x; s) = 0 & (t,x) \in 
      [0,T] \times R\\
      \nabla F(t,x; s) \cdot N = 0 & (t,x) \in [0,T] \times \partial R\\
      F(0,x; s) = f(u(s,x)) & x\in R
    \end{cases}.
  \end{equation}
  
  In the case that $f(u) = (u-a)^2$ then $f'(u) = 2(u-a)$, the PDE for $F$ becomes
  \begin{equation}
    \begin{cases}
      \partial_t F(t,x; s) = \Delta F(t,x; s) & x\in R\times
      [0,T] \\
      \nabla F(t,x; s) \cdot N = 0 & x\in \partial R \times [0,T] \\
      F(0,x; s) = 2(u(s,x)-a) & x\in R
    \end{cases},
  \end{equation}
  which is the forward Heat Equation with initial condition being the solution of the same Heat Equation evaluated at time $s$. By the semi-group property of the Heat Equation, we have that
  \begin{equation}
    F(t,x; s) = 2( u(s+t,x)-a ),
  \end{equation}
  and therefore using \eqref{eq:lambda}, 
  \begin{equation}
    \lambda(t,x) = -2\int_t^T ( u(s + s - t,x ) - a ) \ud s = -2\int_t^T ( u(2s - t,x ) - a )\ud s.
  \end{equation}
\end{proof}

\begin{proposition}
  The gradient of $E$ with respect to the boundary $\partial R$ can be expressed as
  \begin{equation} \label{eq:grad_energy_2}
    \nabla_{\partial R} E = \int_0^T
    \left[ f(u) + \nabla \lambda \cdot \nabla u +
      \lambda \partial_t u \right] \ud t \cdot N,
  \end{equation}
  where $N$ is the normal vector to $\partial R$. In the case that $f(u) = (u-a)^2$ and as $T$ gets large, the gradient approaches
  \begin{equation} \label{eq:grad_simple_2}
    \nabla_{\partial R} E = \left(-\frac 1 2 |\nabla \lambda(0) |^2 -
      \lambda(0)[u(0)-a]\right) N,  \quad
    \lambda(0,x) = -\int_0^{2T} (u(s,x) -a) \ud s,
  \end{equation}
  where $\lambda(0)$ and $u(0)$ denote the functions $\lambda$ and $u$ at
  time zero.
\end{proposition}
\begin{proof}
  To compute the gradient of $E$, we compute the gradient of $E$ in \eqref{eq:energy_lagrange_2} with respect to $\partial R$ treating $\lambda$ and $u$ independent of $R$ as in the theory of Lagrange multipliers. In this case, this is just a classical result in the calculus of variations (e.g., \cite{zhu1996region}), in particular the integrand (with respect to $R$) is multiplied by the outward normal along $\partial R$ to obtain the gradient:
  \begin{equation}
    \nabla_{\partial R} E = \int_0^T
    \left[ f(u) + \nabla \lambda \cdot \nabla u +
      \lambda \partial_t u \right] \ud t \cdot N.
  \end{equation}  
  
  Note that using a change of variables $\tau = 2s - t$, we may write
  $\lambda$ in \eqref{eq:lambda_explicit_2} as
  \begin{equation}
    \lambda(t,x) = -\int_{t}^{2T-t} ( u(\tau,x) - a ) \ud \tau, \quad t\in[0,T],
  \end{equation}
  as in \eqref{eq:grad_simple_2}. 

  If $f(u) = (u-a)^2$, then we may write $\nabla E = F N$ where 
  \begin{equation}
  F = \int_0^T (u-a)^2 + \nabla \lambda \cdot \nabla u +
  \lambda \partial_tu \ud t.
  \end{equation}
  Integrating by parts in $t$ yields
  \begin{equation}
    F = \int_0^T (u-a)^2 + \nabla \lambda \cdot \nabla u -
    \partial_t \lambda u \ud t -\lambda(0) u(0).
  \end{equation}  
  If we let $T\to\infty$, then 
  \begin{equation}
  \lambda(t,x) = -\int_{t}^{\infty} ( u(\tau,x) - a ) \ud \tau,
\end{equation}
  and so 
  \begin{align}
    F &= \int_0^{\infty} (u-a)^2 + \nabla \lambda \cdot \nabla u -
        u(u-a_i) \ud t -\lambda(0) u(0) \\
      & = 
        \int_0^{\infty} -a(u-a) + \nabla \lambda \cdot \nabla u \ud t - 
        \lambda(0) u(0),
  \end{align}
  where we used integration by parts and noted that
  $\lambda(\infty)=0$. Therefore,
  \begin{equation} \label{eq:F_simp}
  F =  \int_0^{\infty} \nabla \lambda \cdot \nabla u \ud t -
  \lambda(0)( u(0) - a ),
  \end{equation}
  by noting the first term is $a\lambda(0)$. We may now simplify the integral above:
  \begin{align}
    \int_0^{\infty} \nabla \lambda(t,x) \cdot \nabla u(t,x) \ud t &= -\int_0^{\infty}
    \int_{t}^{\infty} \nabla u(\tau,x) \cdot \nabla u(t,x) \ud \tau \ud t \\
     &= 
       -\frac 1 2 \int_0^{\infty}
       \int_{0}^{\infty} \nabla u(\tau,x) \cdot \nabla u(t,x)\ud \tau \ud t \\
    &=
      -\frac 1 2 \int_0^{\infty} \nabla u(\tau,x) \ud \tau \cdot 
      \int_{0}^{\infty} \nabla u(t,x)\ud t
  \end{align}
  where we have used symmetry of the integrand in the second line above. Therefore, 
  \begin{equation}
  \int_0^{\infty} \nabla \lambda(t,x) \cdot \nabla u(t,x) \ud t = -\frac 1 2
  \left| \int_0^{\infty} \nabla u(t,x) \ud t \right|^2 = -\frac 1 2 |\nabla \lambda(0) |^2,
\end{equation}
  and thus substituting into \eqref{eq:F_simp}, we find that 
  \begin{equation}
    F = -\frac 1 2 |\nabla \lambda(0) |^2 - \lambda(0)( u(0) - a ).
  \end{equation}

\end{proof}

\bibliographystyle{splncs}
\bibliography{scale_refs}

\end{document}